\documentclass[11pt]{article}

\usepackage{fancyhdr}
\usepackage{titlesec}
\usepackage{titling}
\usepackage{wrapfig}
\usepackage{amsmath}
\usepackage{amsfonts}
\usepackage{amssymb}
\usepackage{amsthm}
\usepackage{epsfig}
\usepackage{natbib}
\usepackage[utf8]{inputenc} % allow utf-8 input
\usepackage[T1]{fontenc}    % use 8-bit T1 fonts
\usepackage{hyperref}       % hyperlinks
\usepackage{url}            % simple URL typesetting
\usepackage{booktabs}       % professional-quality tables
\usepackage[leftcaption]{sidecap}
\usepackage{nicefrac}       % compact symbols for 1/2, etc.
\usepackage{microtype}      % microtypography
\usepackage{graphicx}
\usepackage{paralist}
\usepackage{enumitem}
\usepackage{xcolor}
\usepackage{comment}
\usepackage{mathtools}
\usepackage{bbm}
\usepackage[linesnumbered,ruled]{algorithm2e}
\usepackage{algorithmic}
\usepackage{thm-restate}
\usepackage{capt-of}
\usepackage{authblk}
\usepackage[symbol,flushmargin]{footmisc}
\usepackage{dblfloatfix}
\usepackage[font={footnotesize}]{caption}
\usepackage{subcaption}
\usepackage{float}
\usepackage{colortbl}
\usepackage{multirow}

\definecolor{LightCyan}{rgb}{0.88,1,1}
\definecolor{LightOrange}{rgb}{1,0.85,0.70}
\definecolor{DarkCyan}{rgb}{0,0.8,0.8}
\definecolor{DarkOrange}{rgb}{1,0.50,0.0}

\newtheorem{lemma}{Lemma}

\newcommand\numberthis{\addtocounter{equation}{1}\tag{\theequation}}

\usepackage[makeroom]{cancel}

\usepackage{mathtools,stackengine}
\stackMath
\newcommand{\stackEq}[1]{%
  \setbox0=\hbox{${}\mathrel{\stackon[-1pt]{=}{\scriptstyle\text{#1\strut}}}{}$}
  \xdef\tmpwd{\dimexpr\the\wd0\relax}
  \kern.5\tmpwd\mathclap{\box0}&\kern.5\tmpwd
}

\def\eqref#1{equation~\ref{#1}}
\def\1{\bm{1}}

\DeclareMathAlphabet{\mathsfit}{\encodingdefault}{\sfdefault}{m}{sl}
\SetMathAlphabet{\mathsfit}{bold}{\encodingdefault}{\sfdefault}{bx}{n}

\def\gB{{\mathcal{B}}}

\def\gE{{\mathcal{E}}}

\def\gH{{\mathcal{H}}}

\def\gM{{\mathcal{M}}}

\def\gX{{\mathcal{X}}}

\newcommand{\E}{\mathbb{E}}

\newcommand{\mc}[1]{{\mathcal{#1}}}

\newcommand{\RR}{\mathbb{R}}

\renewcommand{\S}{\mathcal{S}}
\newcommand{\A}{\mathcal{A}}

\newcommand{\X}{\mathcal{X}}
\newcommand{\Y}{\mathcal{Y}}

\newcommand{\grad}{\nabla_{\theta}}

\renewcommand{\cite}{\citep}

\makeatletter
\renewcommand\AB@affilsepx{}
\makeatother

\begin{document}

\title{{\LARGE\bfseries Imitation with Neural Density Models}}

\date{}

\author[1]{Kuno Kim}
\author[1]{Akshat Jindal}
\author[1]{Yang Song}
\author[1]{Jiaming Song}
\author[2]{Yanan Sui}
\author[1]{Stefano Ermon}

\affil[1]{Department of Computer Science, Stanford University \quad} 
\affil[2]{Department of Computer Science, Tsinghua University \newline}

\maketitle
\vspace{-3cm}
 
\thispagestyle{empty}

\begin{abstract}
We propose a new framework for Imitation Learning (IL) via density estimation of the expert's occupancy measure followed by Maximum Occupancy Entropy Reinforcement Learning (RL) using the density as a reward. Our approach maximizes a non-adversarial model-free RL objective that provably lower bounds reverse Kullback–Leibler divergence between occupancy measures of the expert and imitator. We present a practical IL algorithm, Neural Density Imitation (NDI), which obtains state-of-the-art demonstration efficiency on benchmark control tasks. 
\end{abstract}

\section{Introduction} 
\label{section:introduction}

Imitation Learning (IL) algorithms aim to learn optimal behavior by mimicking expert demonstrations. Perhaps the simplest IL method is Behavioral Cloning (BC) \cite{pomerleau1991efficient} which ignores the dynamics of the underlying Markov Decision Process (MDP) that generated the demonstrations, and treats IL as a supervised learning problem of predicting optimal actions given states. Prior work showed that if the learned policy incurs a small BC loss, the worst case performance gap between the expert and imitator grows quadratically with the number of decision steps \cite{ross2010efficient, ross2011a}. The crux of their argument is that policies that are "close" as measured by BC loss can induce disastrously different distributions over states when deployed in the environment. One family of solutions to mitigating such compounding errors is Interactive IL \cite{ross2011reduction, ross2013uav, guo2014ataridagger}, which involves running the imitator's policy and collecting corrective actions from an interactive expert. However, interactive expert queries can be expensive and are seldom available. 

Another family of approaches \cite{ho2016generative, fu2017learning, ke2019divergenceimitation, kostrikov2020valuedice, kim2018mmdgail, wang2017diverse} that have gained much traction is to directly minimize a statistical distance between state-action distributions induced by policies of the expert and imitator, i.e the occupancy measures $\rho_{\pi_{E}}$ and $\rho_{\pi_{\theta}}$. 
%Informally, $\rho_{\pi_{\theta}}$ is proportional to how frequently $\pi_{\theta}$ visits certain state-action pairs when executed in the environment. 
As $\rho_{\pi_{\theta}}$ is an implicit distribution induced by the policy and environment\footnote{we assume only samples can be taken from the environment dynamics and its density is unknown}, distribution matching with $\rho_{\pi_{\theta}}$ typically requires likelihood-free methods involving sampling. Sampling from $\rho_{\pi_{\theta}}$ entails running the imitator policy in the environment, which was not required by BC. While distribution matching IL requires additional access to an environment simulator, it has been shown to drastically improve demonstration efficiency, i.e the number of demonstrations needed to succeed at IL \cite{ho2016generative}. A wide suite of distribution matching IL algorithms use adversarial methods to match $\rho_{\pi_{\theta}}$ and $\rho_{\pi_{E}}$, which requires alternating between reward (discriminator) and policy (generator) updates \cite{ho2016generative, fu2017learning, ke2019divergenceimitation, kostrikov2020valuedice, kim2019cross}. A key drawback to such Adversarial Imitation Learning (AIL) methods is that they inherit the instability of alternating min-max optimization \cite{salimans2016gantechniques, miyato2018spectral} which is generally not guaranteed to converge \cite{jin2019minmax}. Furthermore, this instability is exacerbated in the IL setting where generator updates involve high-variance policy optimization and leads to sub-optimal demonstration efficiency. To alleviate this instability, \cite{wang2020red, brantley2020disagreement, reddy2017sqil} have proposed to do RL with fixed heuristic rewards. \citet{wang2020red}, for example, uses a heuristic reward that estimates the support of $\rho_{\pi_{E}}$ which discourages the imitator from visiting out-of-support states. While having the merit of simplicity, these approaches have no guarantee of recovering the true expert policy. 

In this work, we propose a new framework for IL via obtaining a density estimate $q_{\phi}$ of the expert's occupancy measure $\rho_{\pi_E}$ followed by Maximum Occupancy Entropy Reinforcement Learning (MaxOccEntRL) \cite{lee2019marginal, islam2019maxent}. In the MaxOccEntRL step, the density estimate $q_{\phi}$ is used as a \emph{fixed reward} for RL and the occupancy entropy $\mc{H}(\rho_{\pi_{\theta}})$ is simultaneously maximized, leading to the objective $\max_{\theta} \E_{\rho_{\pi_{\theta}}}[q_{\phi}(s, a)] + \mc{H}(\rho_{\pi_{\theta}}) $. Intuitively, our approach encourages the imitator to visit high density state-action pairs under $\rho_{\pi_{E}}$ while maximally exploring the state-action space. 
There are two main challenges to this approach. First, we require accurate density estimation of $\rho_{\pi_{E}}$, which is particularly challenging when the state-action space is high dimensional and the number of expert demonstrations are limited. Second, in contrast to Maximum Entropy RL (MaxEntRL), MaxOccEntRL requires maximizing the entropy of an implicit density $\rho_{\pi_{\theta}}$. We address the former challenge leveraging advances in density estimation \cite{germain2015made, du2018ebm, song2019sliced}. For the latter challenge, we derive a non-adversarial model-free RL objective that provably maximizes a lower bound to occupancy entropy. As a byproduct, we also obtain a model-free RL objective that lower bounds reverse Kullback-Lieber (KL) divergence between $\rho_{\pi_{\theta}}$ and $\rho_{\pi_{E}}$. We evaluate our method, named Neural Density Imitation (NDI), on high-dimensional benchmark robotics tasks, and show that it achieves state-of-the-art demonstration efficiency. 

\section{Imitation Learning via density estimation}
\label{section:theory}
We model an agent's decision making process as a discounted infinite-horizon Markov Decision Process (MDP) $\gM = (\S, \A, P, P_{0}, r, \gamma)$. Here $\S, \A$ are state-action spaces, $P: \S \times \A \rightarrow \Omega(\S)$ is a transition dynamics where $\Omega(\S)$ is the set of probability measures on $\S$, $P_{0}: \S \rightarrow \RR$ is an initial state distribution, $r: \S \times \A \rightarrow \RR$ is a reward function, and $\gamma \in [0, 1)$ is a discount factor. A parameterized policy $\pi_{\theta}: \S \rightarrow \Omega(\A)$ distills the agent's decision making rule and $\{s_{t}, a_{t}\}_{t=0}^{\infty}$ is the stochastic process realized by sampling an initial state from $s_{0} \sim P_{0}(s)$ then running $\pi_{\theta}$ in the environment, i.e $a_{t} \sim \pi_{\theta}(\cdot | s_{t}), s_{t+1} \sim P(\cdot |s_{t}, a_{t})$. We denote by $p_{\theta, t:t+k}$ the joint distribution of states $\{s_{t}, s_{t+1}, ..., s_{t+k}\}$, where setting $p_{\theta, t}$ recovers the marginal of $s_{t}$. The (unnormalized) occupancy measure of $\pi_{\theta}$ is defined as $\rho_{\pi_{\theta}}(s, a) = \sum_{t = 0}^{\infty} \gamma^{t} p_{\theta, t}(s)\pi_{\theta}(a | s)$. Intuitively, $\rho_{\pi_{\theta}}(s, a)$ quantifies the frequency of visiting the state-action pair $(s, a)$ when running $\pi_{\theta}$ for a long time, with more emphasis on earlier states. 

We denote policy performance as $J(\pi_{\theta}, \bar{r}) = \E_{\pi_{\theta}}[\sum_{t=0}^{\infty} \gamma^{t} \bar{r}(s_t, a_t)] = \E_{(s, a) \sim \rho_{\pi_{\theta}}}[\bar{r}(s, a)]$ where $\bar{r}$ is a (potentially) augmented reward function and $\E$ denotes the generalized expectation operator extended to non-normalized densities $\hat{p}: \mc{X} \rightarrow \RR^{+}$ and functions $f: \mc{X} \rightarrow \mc{Y}$ so that $\E_{\hat{p}}[f(x)] = \sum_x \hat{p}(x)f(x)$. The choice of $\bar{r}$ depends on the RL framework. In standard RL, we simply have $\bar{r} = r$, while in Maximum Entropy RL (MaxEntRL) \cite{haarnoja2017reinforcement}, we have $\bar{r}(s, a) = r(s, a) - \log \pi_{\theta}(a | s)$. We denote the entropy of $\rho_{\pi_{\theta}}(s, a)$ as $\gH(\rho_{\pi_{\theta}}) = \E_{\rho_{\pi_{\theta}}}[-\log \rho_{\pi_{\theta}}(s, a)]$ and overload notation to denote the $\gamma$-discounted causal entropy of policy $\pi_{\theta}$ as $\gH(\pi_{\theta}) = \E_{\pi_{\theta}}[\sum_{t=0}^{\infty} -\gamma^{t} \log \pi_{\theta}(a_{t} | s_{t})] = \E_{\rho_{\pi_{\theta}}}[-\log \pi_{\theta}(a | s)]$. Note that we use a generalized notion of entropy where the domain is extended to non-normalized densities. 
We can then define the Maximum Occupancy Entropy RL (MaxOccEntRL) \cite{lee2019marginal, islam2019maxent} objective as $J(\pi_{\theta}, \bar{r}=r) + \gH(\rho_{\pi_{\theta}})$. Note the key difference between MaxOccEntRL and MaxEntRL: entropy regularization is on the occupancy measure instead of the policy, i.e seeks state diversity instead of action diversity. We will later show in section \ref{sec:maxoccentrl}, that a lower bound on this objective reduces to a complete model-free RL objective with an augmented reward  $\bar{r}$. 

Let $\pi_{E}, \pi_{\theta}$ denote an expert and imitator policy, respectively. Given only demonstrations $\mc{D} = \{(s, a)_{i}\}_{i = 1}^{k} \sim \pi_{E}$ of state-action pairs sampled from the expert, Imitation Learning (IL) aims to learn a policy $\pi_{\theta}$ which matches the expert, i.e $\pi_{\theta} = \pi_{E}$. Formally, IL can be recast as a distribution matching problem \cite{ho2016generative, ke2019divergenceimitation} between occupancy measures $\rho_{\pi_{\theta}}$ and $\rho_{\pi_{E}}$:
\begin{align*}
    \mathrm{maximize}_{\theta} -d(\rho_{\pi_{\theta}}, \rho_{\pi_{E}}) \numberthis \label{eq:dist_match}
\end{align*}
where $d(\hat{p}, \hat{q})$ is a generalized statistical distance defined on the extended domain of (potentially) non-normalized probability densities $\hat{p}(x), \hat{q}(x)$ with the same normalization factor $Z > 0$, i.e $\int_x \hat{p}(x)/Z = \int_x \hat{q}(x)/Z = 1$. For $\rho_{\pi}$ and $\rho_{\pi_{E}}$, we have $Z = \frac{1}{1 - \gamma}$. As we are only able to take samples from the transition kernel and its density is unknown, $\rho_{\pi_{\theta}}$ is an implicit distribution\footnote{probability models that have potentially intractable density functions, but can be sampled from to estimate expectations and gradients of expectations with respect to model parameters \cite{huszar2017implicit}.}.
Thus, optimizing Eq. \ref{eq:dist_match} typically requires likelihood-free approaches leveraging samples from $\rho_{\pi_{\theta}}$, i.e running $\pi_{\theta}$ in the environment. Current state-of-the-art IL approaches use likelihood-free adversarial methods to approximately optimize Eq. \ref{eq:dist_match} for various choices of $d$ such as reverse Kullback-Liebler (KL) divergence \cite{fu2017learning, kostrikov2020valuedice} and Jensen-Shannon (JS) divergence \cite{ho2016generative}. However, adversarial methods are known to suffer from optimization instability which is exacerbated in the IL setting where one step in the alternating optimization involves RL. 

We instead derive a non-adversarial objective for IL. In this work, we choose $d$ to be (generalized) reverse-KL divergence and leave derivations for alternate choices of $d$ to future work.
\begin{align*}
-D_{\mathrm{KL}}(\rho_{\pi_{\theta}} || \rho_{\pi_{E}}) 
= \E_{\rho_{\pi_{\theta}}}[\log \rho_{\pi_{E}}(s, a) - \log \rho_{\pi_{\theta}}(s, a)] = J(\pi_{\theta}, \bar{r} = \log \rho_{\pi_{E}}) + \gH(\rho_{\pi_{\theta}}) \numberthis \label{eq:reverseKL} 
\end{align*}
We see that maximizing negative reverse-KL with respect to $\pi_{\theta}$ is equivalent to Maximum Occupancy Entropy RL (MaxOccEntRL) with $\log \rho_{\pi_{E}}$ as the fixed reward. Intuitively, this objective drives $\pi_{\theta}$ to visit states that are most likely under $\rho_{\pi_{E}}$ while maximally spreading out probability mass so that if two state-action pairs are equally likely, the policy visits both. There are two main challenges associated with this approach which we address in the following sections. 
\begin{enumerate}[leftmargin=*]
\item $\log \rho_{\pi_{E}}$ is unknown and must be estimated from the demonstrations $\mc{D}$. Density estimation remains a challenging problem, especially when there are a limited number of samples and the data is high dimensional \cite{liu2007nonparametric}. Note that simply extracting the conditional $\pi(a|s)$ from an estimate of the joint $\rho_{\pi_{E}}(s, a)$ is an alternate way to do BC and does not resolve the compounding error problem \cite{ross2011a}. 

\item $\gH(\rho_{\pi_{\theta}})$ is hard to maximize as $\rho_{\pi_{\theta}}$ is an implicit density. This challenge is similar to the difficulty of entropy regularizing generators \cite{mohamed2016implicit, belghazi2018entropy, dieng2019prescribed} for Generative Adversarial Networks (GANs) \cite{goodfellow2014generative}, and most existing approaches \cite{dieng2019prescribed, lee2019marginal} use adversarial optimization.
\end{enumerate}

\subsection{Estimating the expert occupancy measure} 
We seek to learn a parameterized density model $q_{\phi}(s, a)$ of $\rho_{\pi_{E}}$ from samples. We consider two canonical families of density models: Autoregressive models and Energy-based models (EBMs). 

\textbf{Autoregressive Models} \cite{germain2015made, papamakarios2017maf}: An autoregressive model $q_{\phi}(x)$ for $x = (x_1, ..., x_{\mathrm{dim}(\S) + \mathrm{dim}(\A)}) = (s, a)$ learns a factorized distribution of the form: 
\begin{equation}
q_{\phi}(x) = \Pi_i q_{\phi_i}(x_i | x_{<i})
\label{eq:made}
\end{equation}
For instance, each factor $q_{\phi_{i}}$ could be a mapping from $x_{<i}$ to a Gaussian density over $x_{i}$. When given a prior over the true dependency structure of $\{x_{i}\}$, this can be incorporated by refactoring Eq. \ref{eq:made}. Autoregressive models are typically trained via Maximum Likelihood Estimation (MLE). 
\begin{equation} 
\max_{\phi} \E_{\rho_{\pi_{E}}}[\log q_{\phi}(s, a)]
\label{eq:mle}
\end{equation}

\textbf{Energy-based Models (EBM)} \cite{du2018ebm, song2019sliced}: Let $\gE_{\phi}: \S \times \A \rightarrow \RR$ be an energy function. An energy based model is a parameterized Boltzman distribution of the form: 
\begin{equation}
q_{\phi}(s, a) = \frac{1}{Z(\phi)} e^{-\gE_{\phi}(s, a)} 
\label{eq:ebm}
\end{equation}
where $Z(\phi) = \int_{\S \times \A} e^{-\gE_{\phi}(s, a)}$ denotes the partition function. Energy-based models are desirable for high dimensional density estimation due to their expressivity, but are typically difficult to train due to the intractability of computing the partition function. However, our IL objective in Eq. \ref{eq:dist_match} conveniently only requires a non-normalized density estimate as policy optimality is invariant to constant shifts in the reward.
Thus, we opted to perform non-normalized density estimation with EBMs using score matching which allows us to directly learn $\gE_{\phi}$ without having to estimate $Z(\phi)$. Standard score matching \cite{hyvarinen2005estimation} minimizes the following objective.
\begin{equation}
\min_{\phi} \E_{\rho_{\pi_{E}}}[\mathrm{tr}(\nabla^2_{s, a} \gE_{\phi}(s, a)) + \frac{1}{2}\lVert \nabla_{s, a} \gE_{\phi}(s, a) \rVert_{2}^{2}]
\label{eq:score}
\end{equation}

\subsection{Maximum Occupancy Entropy Reinforcement Learning}
\label{sec:maxoccentrl}

In general maximizing the entropy of implicit distributions is challenging due to the fact that there is no analytic form for the density function. Prior works have proposed using adversarial methods involving noise injection \cite{dieng2019prescribed} and fictitious play \cite{brown1951fictitious, lee2019marginal}. We instead propose to maximize a novel lower bound to the occupancy entropy which we prove is equivalent to maximizing a non-adversarial model-free RL objective. We first introduce a crucial ingredient in deriving our occupancy entropy lower bound, which is a tractable lower bound to Mutual Information (MI) first proposed by Nguyen, Wainright, and Jordan \cite{nguyen2010nwj}, also known as the $f$-GAN KL \cite{nowozin2016f} and MINE-$f$ \cite{belghazi2018entropy}. For random variables $X, Y$ distributed according to $p_{\theta_{xy}}(x, y), p_{\theta_{x}}(x), p_{\theta_{y}}(y)$ where $\theta = (\theta_{xy}, \theta_{x}, \theta_{y})$, and any critic function $f: \X \times \Y \rightarrow \RR$, it holds that $I(X; Y | \theta) \geq I^{f}_{\mathrm{NWJ}}(X; Y | \theta)$ where, 
\begin{equation}
I^f_{\mathrm{NWJ}}(X; Y | \theta) := \E_{p_{\theta_{xy}}}[f(x, y)] - e^{-1}\E_{p_{\theta_{x}}}[\E_{p_{\theta_{y}}}[e^{f(x, y)}]]
\label{eq:inwj}
\end{equation}
This bound is tight when $f$ is chosen to be the optimal critic $f^*(x, y) = \log \frac{p_{\theta_{xy}}(x, y)}{p_{\theta_{x}}(x)p_{\theta_y}(y)} + 1$. We are now ready to state a lower bound to the occupancy entropy. 

\begin{restatable}{theorem}{lowerbound} 
\label{thm:entropy_lowerbound}
Let MDP $\gM$ satisfy assumption \ref{ass:mdp} (App. \ref{sec:proofs}). For any critic $f: \S \times \S \rightarrow \RR$, it holds that 
\begin{equation}
\gH(\rho_{\pi_{\theta}}) 
\geq \gH^{f}(\rho_{\pi_{\theta}})
\label{eq:entropy_lowerbound}
\end{equation}
where 
\begin{equation}
\gH^{f}(\rho_{\pi_{\theta}}) 
:= \gH(s_{0}) 
+ (1 + \gamma) \gH(\pi_{\theta}) 
+ \gamma \sum_{t = 0}^{\infty} \gamma^{t} I^{f}_{\mathrm{NWJ}}(s_{t+1}; s_{t} | \theta)
\label{eq:entropy_def}
\end{equation}
\end{restatable}
\vspace{-10pt}
\begin{comment}
\se{notation could be a bit deceiveing, because it doesn't show the last term depends on theta too}
\se{note for future: we have 2 sources of slackness in the bound. for a sufficiently flexible f, we could get rid of one (INWJ is tight). for the other one, the bound would be tight if the marginal state-distributions over time are the same. we know that they eventually become the same if the markov chain mixes. so it might be possible to quantify the slack. e.g. it'd expect that bound to be tight as gamma goes to zero so the "mixing phase" doesn't matter much} \kh{I think it would be good to include some discussion about the tightness of the bound} 
\end{comment}

See Appendix \ref{sec:proof_thm_entropy_lowerbound} for the proof and a discussion of the bound tightness. Here onwards, we refer to $\gH^{f}(\rho_{\pi_{\theta}})$ from Theorem \ref{thm:entropy_lowerbound} as the State-Action Entropy Lower Bound (SAELBO). The SAELBO mainly decomposes into policy entropy and mutual information between consecutive states. Intuitively, the policy entropy term $\gH(\pi_{\theta})$ pushes the policy to be as uniformly random as possible. The MI term $I_{\mathrm{NWJ}}(s_{t+1}; s_{t} | \theta)$ encourages the agent to seek out states where it is currently behaving least randomly as the next state is more predictable from the current state (i.e shares more information) if the policy is less random. Together these objectives will encourage the agent to continuously seek out states where its policy is currently least random and update the policy to be more random in those states!
We note an important distinction between solely maximizing the policy entropy term versus our bound. Consider a video game where the agent must make a sequence of purposeful decisions to unlock new levels, e.g Montezuma's revenge. Assume that your policy is a conditional gaussian with the initial variance set to be small. Soley optimizing policy entropy will yield a uniform random policy which has very low probability of exploring any of the higher levels. On the otherhand, optimizing our bound will encourage the agent to continuously seek out novel states, including new levels, since the policy variance will be small there. Next, we show that the gradient of the SAELBO is equivalent to the gradient of a model-free RL objective. 

\begin{restatable}{theorem}{entropygradient} 
\label{thm:gradient}
Let $q_{\pi}(a | s)$ and $\{q_{t}(s)\}_{t \geq 0}$ be probability densities such that $\forall s, a \in \S \times \A$ satisfy $q_{\pi}(a | s) = \pi_{\theta}(a | s)$ and $q_{t}(s) = p_{\theta, t}(s)$. Then for all $f: \S \times \S \rightarrow \RR$, 
\begin{equation}
\grad \gH^{f}(\rho_{\pi_{\theta}}) = \grad J(\pi_{\theta}, \bar{r} = r_{\pi} + r_{f})
\label{eq:entropy_gradient}
\end{equation}
where
\begin{align*}
r_{\pi}(s_t, a_t) &= -(1+\gamma) \log q_{\pi}(a_t | s_t) \numberthis \label{eq:reward_pi}\\
r_{f}(s_{t}, a_{t}, s_{t+1}) &= \gamma f(s_{t}, s_{t+1}) -  \frac{\gamma}{e}\E_{\tilde{s}_{t} \sim q_{t}, \tilde{s}_{t+1} \sim q_{t+1}}[e^{f(\tilde{s}_{t}, s_{t+1})} + e^{f(s_{t}, \tilde{s}_{t+1})}] \numberthis \label{eq:reward_f}
\end{align*}
\end{restatable}

See Appendix \ref{sec:proof_thm_gradient} for the proof. Theorem \ref{thm:gradient} shows that maximizing the SAELBO is equivalent to maximizing a discounted model-free RL objective with the reward $r_{\pi} + r_{f}$, where $r_{\pi}$ contributes to maximizing $\gH(\pi_{\theta})$ and $r_{f}$ contributes to maximizing $\sum_{t = 0}^{\infty} \gamma^{t} I^{f}_{\mathrm{NWJ}}(s_{t+1}; s_{t} | \theta)$. Note that evaluating $r_{f}$ entails estimating expectations with respect to $q_{t}, q_{t+1}$. This can be accomplished by rolling out multiple trajectories with the current policy and collecting the states from time-step $t, t+1$. Alternatively, if we assume that the policy is changing slowly, we can simply take samples of states from time-step $t, t+1$ from the replay buffer. Combining the results of Theorem \ref{thm:entropy_lowerbound}, \ref{thm:gradient}, we end the section with a lower bound on the original distribution matching objective from Eq. \ref{eq:dist_match} and show that maximizing this lower bound is again, equivalent to maximizing a model-free RL objective. 

\begin{restatable}{corollary}{klgradient} 
\label{cor:objective}
Let MDP $\gM$ satisfy assumption \ref{ass:mdp} (App. \ref{sec:proofs}). For any critic $f: \S \times \S \rightarrow \RR$, it holds that
\begin{equation}
-D_{\mathrm{KL}}(\rho_{\pi_\theta} || \rho_{\pi_{E}}) \geq J(\pi_\theta, \bar{r}=\log \rho_{\pi_{E}}) + \gH^{f}(\rho_{\pi_{\theta}})
\end{equation}
Furthermore, let $r_{\pi}, r_{f}$ be defined as in Theorem \ref{thm:gradient}. Then, 
\begin{align*}
\grad \big(J(\pi_\theta, \bar{r}=\log \rho_{\pi_{E}}) + \gH^{f}(\rho_{\pi_{\theta}})\big) = \grad J(\pi_{\theta}, \bar{r}=\log \rho_{\pi_{E}}+r_{\pi}+r_{f}) \numberthis 
\label{eq:objective}
\end{align*}
\end{restatable}
In the following section we derive a practical distribution matching IL algorithm combining all the ingredients from this section.

\begin{algorithm}[t]
\caption{Neural Density Imitation (NDI)} \label{alg:density}
\textbf{Require:} Demonstrations $\mc{D} \sim \pi_{E}$, Reward weights $\lambda_{\pi}, \lambda_{f}$, Learning rates $\eta_{\theta}, \eta_{\phi}$, Critic $f$ \\ 
\vspace{7pt}
\emph{\textbf{Phase 1}. Density estimation}: \\
Learn $q_{\phi}(s, a)$ from $\mc{D}$ using MADE or EBMs by optimizing Eq. \ref{eq:mle} or \ref{eq:ssm_objective} with learning rate $\eta_{\phi}$. \\
\vspace{7pt}
\emph{\textbf{Phase 2}. MaxOccEntRL}:
\\
\For{$k = 1, 2, ...$}{
    Collect $(s_t, a_t, s_{t+1}, \bar{r}) \sim \pi_{\theta}$ and add to replay buffer $\mc{B}$, where $\bar{r}=\log q_{\phi}+\lambda_{\pi}r_{\pi}+\lambda_{f}r_{f}$,
    \begin{align*}
        r_{\pi}(s_t, a_t) &= -\log \pi_{\theta}(a_{t} | s_{t}) \\
        r_{f}(s_{t}, a_{t}, s_{t+1}) &= f(s_{t}, s_{t+1}) -  \frac{\gamma}{e}\E_{\tilde{s}_t \sim \gB_t, \tilde{s}_{t+1}\sim \gB_{t+1}}[e^{f(s_{t+1}, \tilde{s}_{t})} + e^{f(\tilde{s}_{t+1}, s_{t})}]
    \end{align*}
    
    Update $\pi_{\theta}$ using Soft Actor-Critic (SAC) \cite{haarnoja2018soft}: 
    \begin{align*}
    \theta_{k+1} \leftarrow \theta_{k} + \eta_{\theta}\grad J(\theta, \bar{r}) |_{\theta = \theta_{k}}
    \end{align*}
    \vspace{-15pt}
}
\label{alg:awml}
\end{algorithm}

\section{Neural Density Imitation (NDI)}
\label{section:algorithm}

From previous section's results, we propose Neural Density Imitation (NDI) that works in two phases:  

\textbf{Phase 1: Density estimation}: As described in Section \ref{section:theory}, we seek to leverage Autoregressive models and EBMs for density estimation of the expert's occupancy measure $\rho_{\pi_{E}}$ from samples. As in \cite{ho2016generative, fu2017learning}, we take the state-action pairs in the demonstration set $\mc{D} = \{(s, a)_i\}_{i = 1}^{N} \sim \pi_{E}$ to approximate samples from $\rho_{\pi_{E}}$ and fit $q_{\phi}$ on $\mc{D}$. For Autoregressive models, we use Masked Autoencoders for Density Estimation (MADE) \cite{germain2015made} where the entire collection of conditional density models $\{q_{\phi_i}\}$ is parameterized by a single masked autoencoder network. Specifically, we use a gaussian mixture variant \cite{papamakarios2017maf} of MADE where each of the conditionals $q_{\phi_i}$ map inputs $x_{<i}$ to the mean and covariance of a gaussian mixture distribution over $x_{i}$. The autoencoder is trained via MLE as in Eq. \ref{eq:mle}. With EBMs, we perform non-normalized log density estimation and thus directly parameterize the energy function $\gE_{\phi}$ with neural networks since $\log q_{\phi} = \gE_{\phi} + \log Z(\phi)$. One hurdle to learning $\gE_{\phi}$ with the standard score matching objective of Eq. \ref{eq:score} is the computational expense of estimating $\mathrm{tr}(\nabla^2_{s, a} \gE_{\phi}(s, a))$. To reduce this expense, we use Sliced Score Matching \cite{song2019sliced} which leverages random projections $v \sim p_v$ to approximate the trace term as in the following equation. 
\vspace{-5pt}
\begin{equation} 
\min_{\phi} \E_{p_{v}}[\E_{\rho_{\pi_{E}}}[v^{T}\nabla^2_{s, a} \gE_{\phi}(s, a) v + \frac{1}{2}\lVert \gE_{\phi}(s, a) \rVert_{2}^{2}]]
\label{eq:ssm_objective}
\end{equation}
\vspace{-5pt}

\textbf{Phase 2: MaxOccEntRL}
After we've acquired a log density estimate $\log q_{\phi}$ from the previous phase, we perform RL with entropy regularization on the occupancy measure. Inspired by Corollary \ref{cor:objective}, we propose the following RL objective 
\begin{equation}
\max_{\theta} J(\pi_{\theta}, \bar{r}=\log q_{\phi} +\lambda_{\pi}r_{\pi}+\lambda_{f}r_{f})
\label{eq:practical_objective}
\end{equation}
where $\lambda_{\pi}, \lambda_{f} > 0$ are weights introduced to control the influence of the occupancy entropy regularization. In practice, Eq. \ref{eq:practical_objective} can be maximized using any RL algorithm by simply setting the reward function to be $\bar{r}$ from Eq. \ref{eq:practical_objective}. In this work, we use Soft Actor-Critic (SAC) \cite{haarnoja2018soft}. For our critic $f$, we \emph{fix it} to be a normalized RBF kernel for simplicity, but future works could explore learning the critic to match the optimal critic: 
\begin{equation}
f(s_{t+1}, s_{t}) 
= \log \frac{e^{-\lVert s_{t+1} - s_{t}\rVert_2^2}}{\E_{q_t, q_{t+1}}[e^{-\lVert s_{t+1} - s_{t}\rVert_2^2}]} + 1
\label{eq:critic}
\end{equation}
While simple, our choice of $f$ emulates two important properties of the optimal critic $f^*(x, y) = \log \frac{p(x | y)}{p(x)} + 1$: (1). it follows the same "form" of a log-density ratio plus a constant (2). consecutively sampled states from the joint, i.e $s_t, s_{t+1} \sim p_{\theta, t:t+1}$  have high value under our $f$ since they are likely to be close to each other under smooth dynamics, while samples from the marginals $s_t, s_{t+1} \sim q_{t}, q_{t+1}$ are likely to have lower value under $f$ since they can be arbitrarily different states. To estimate the expectations with respect to $q_t, q_{t+1}$ in Eq. \ref{eq:reward_f}, we simply take samples of previously visited states at time $t, t+1$ from the replay buffer. Note that even if $f$ is learned to be the optimal critic, NDI is still non-adversarial since both $\pi$ and $f$ are updated in the same direction of maximizing the objective, unlike GAIL which performs min-max optimization. (see Appendix \ref{sec:adversarial_discussion})

\section{Related Works} 

Prior literature on Imitation learning (IL) in the absence of an interactive expert revolves around Behavioral Cloning (BC) \cite{pomerleau1991efficient, wu2019futurebc} and distribution matching IL \cite{ho2016generative, fu2017learning, song2018multi, ke2019divergenceimitation, kostrikov2020valuedice, kim2019cross}. Many approaches in the latter category minimize statistical divergences using adversarial methods to solve a min-max optimization problem, alternating between reward (discriminator) and policy (generator) updates. 
%As an example, Generative Adversarial Imitation Learning (GAIL) \cite{ho2016generative} minimizes Jensen-Shannon (JS) divergence, Adversarial Inverse Reinforcement Learning (AIRL) \cite{fu2017learning} minimizes reverse KL divergence, and several other divergences \cite{ke2019divergenceimitation, xiao2019wasserstein} have been explored in follow-up work. 
ValueDICE, a more recently proposed adversarial IL approach, formulates reverse KL divergence into a completely off-policy objective thereby greatly reducing the number of environment interactions. A key issue with such Adversarial Imitation Learning (AIL) approaches is optimization instability \cite{miyato2018spectral, jin2019minmax}. Recent works have sought to avoid adversarial optimization by instead performing RL with a heuristic reward function that estimates the support of the expert occupancy measure. Random Expert Distillation (RED) \cite{wang2020red} and Disagreement-regularized IL \cite{brantley2020disagreement} are two representative approaches in this family. A key limitation of these approaches is that support estimation is insufficient to recover the expert policy and thus they require an additional behavioral cloning step. Unlike AIL, we maximize a non-adversarial RL objective and unlike heuristic reward approaches, our objective provably lower bounds reverse KL between occupancy measures of the expert and imitator. 
Density estimation with deep neural networks is an active research area, and much progress has been made towards modeling high-dimensional structured data like images and audio. Most successful approaches parameterize a normalized probability model and estimate it with maximum likelihood, e.g., autoregressive models~\cite{uria2013rnade,uria2016neural,germain2015made,oord2016pixel} and normalizing flow models~\cite{dinh2014nice,dinh2016density,kingma2018glow}. Some other methods explore estimating non-normalized probability models with MCMC~\cite{du2019implicit,yu2020training} or training with alternative statistical divergences such as score matching~\cite{hyvarinen2005estimation, song2019sliced, song2019gradients} and noise contrastive estimation~\cite{gutmann2010noise,gao2019flow}. Related to MaxOccEntRL, recent works \cite{lee2019marginal, hazan2019maxent, islam2019maxent} on exploration in RL have investigated state-marginal occupancy entropy maximization. To do so, \cite{hazan2019maxent}
requires access to a robust planning oracle, while \cite{lee2019marginal} uses fictitious play, an alternative adversarial algorithm that is guaranteed to converge. Unlike these works, our approach maximizes the SAELBO which requires no planning oracle nor min-max optimization, and is trivial to implement with existing RL algorithms.

\section{Experiments}
\label{section:results}

\textbf{Environment}: Following prior work, we run experiments on benchmark Mujoco \cite{todorov2012mujoco, brockman2017openaigym} tasks: Hopper (11, 3), HalfCheetah (17, 6), Walker (17, 6), Ant (111, 8), and Humanoid (376, 17), where the (observation, action) dimensions are noted parentheses. 

\textbf{Pipeline}: We train expert policies using SAC \cite{haarnoja2018soft}. All of our results are averaged across five random seeds where for each seed we randomly sample a trajectory from an expert, perform density estimation, and then MaxOccEntRL. For each seed we save the best imitator as measured by our augmented reward $\bar{r}$ from Eq. \ref{eq:practical_objective} and report its performance with respect to the ground truth reward. We don't perform sparse subsampling on the data as in \cite{ho2016generative} since real world demonstration data typically aren't subsampled to such an extent and using full trajectories was sufficient to compare performance. 

\textbf{Architecture}: We experiment with two variants of our method, NDI+MADE and NDI+EBM, where the only difference lies in the density model. Across all experiments, our density model $q_{\phi}$ is a two-layer MLP with 256 hidden units. For hyperparameters related to the MaxOccEntRL step, $\lambda_{\pi}$ is tuned automatically in the stable-baselines implementation \cite{stable-baselines}, and we set $\lambda_{f} = 0.005$ (see Section \ref{sec:ablation} for ablation studies). For full details on architecture see Appendix \ref{sec:baselines}.

\textbf{Baselines}: We compare our method against the following baselines: 
(1). \emph{Behavioral Cloning (BC)} \cite{pomerleau1991efficient}: learns a policy via direct supervised learning on $\mc{D}$. 
(2). \emph{Random Expert Distillation (RED)} \cite{wang2020red}: estimates the support of the expert policy using a predictor and target network \cite{burda2018rnd}, followed by RL using this heuristic reward. (3). \emph{Generative Adversarial Imitation Learning (GAIL)} \cite{ho2016generative}: on-policy adversarial IL method which alternates reward and policy updates. 
(4). \emph{ValueDICE} \cite{kostrikov2020valuedice}: current state-of-the-art adversarial IL method that works off-policy. See Appendix \ref{sec:baselines} for baseline implementation details.

\begin{table}[t]
    \setlength{\tabcolsep}{6pt}
    \renewcommand{\arraystretch}{1.}
    \small
    \caption{\textbf{Task Performance} when provided with \emph{one} demonstration. NDI (orange rows) outperforms all baselines on all tasks. See Appendix \ref{sec:more_demo} for results with varying demonstrations.}
    \vspace{0pt}
    \label{table:imitation}
    \begin{center}
    \begin{sc}
    \begin{tabular}{lccccc}
    \toprule
    & Hopper & Half-Cheetah & Walker & Ant & Humanoid \\
    \midrule
    Random & $14\pm8$ & $-282\pm80$ & $1\pm5$ & $-70\pm111$ & $123\pm35$ \\
    \midrule \midrule
    BC & $1432 \pm 382$ & $2674 \pm 633$ & $ 1691 \pm 1008$ & $1425 \pm 812$ & $ 353 \pm 171$\\
    RED & $448 \pm 516$ & $ 383 \pm 819$ & $309 \pm 193$ & $910 \pm 175$ & $242 \pm 67$\\
    GAIL & $3261 \pm 533$ & $3017 \pm 531$ & $3957 \pm 253$ & $ 2299 \pm  519$ & $204 \pm 67$\\
    ValueDICE & $ 2749 \pm 571$ & $3456 \pm 401$ & $ 3342 \pm 1514$ & $1016 \pm 313$ & $ 364 \pm 50$\\
    \rowcolor{LightOrange} NDI+MADE & $3288 \pm 94$ & $4119 \pm 71$ & $4518 \pm 127$ & $555 \pm 311$ & $\mathbf{6088 \pm 689}$ \\
    \rowcolor{LightOrange} NDI+EBM & $\mathbf{3458 \pm 210}$ & $\mathbf{4511 \pm 569}$ & $\mathbf{5061 \pm 135}$ & $\mathbf{4293 \pm 431}$ & $5305 \pm 555$ \\
    \midrule \midrule
    Expert & $3567 \pm 4$ & $4142 \pm 132$ & $5006 \pm 472$ & $4362 \pm 827$ & $5417 \pm 2286$ \\
    \bottomrule
    \end{tabular}
    \end{sc}
    \end{center}

\vspace{-20pt}
\end{table}

\subsection{Task Performance}
\label{sec:imitation}

Table \ref{table:imitation} compares the ground truth reward acquired by agents trained with various IL algorithms when \emph{one} demonstration is provided by the expert. (See Appendix \ref{sec:more_demo} for performance comparisons with varying demonstrations) \emph{NDI+EBM achieves expert level performance on all mujoco benchmarks when provided one demonstration and outperforms all baselines on all mujoco benchmarks}. NDI+MADE achieves expert level performance on 4/5 tasks but fails on Ant. We found spurious modes in the density learned by MADE for Ant, and the RL algorithm was converging to these local maxima. 
We found that baselines are commonly unable to solve Humanoid with one demonstration (the most difficult task considered). RED is unable to perform well on all tasks without pretraining with BC as done in \cite{wang2020red}. For fair comparisons with methods that do not use pretraining, we also do not use pretraining for RED. See Appendix \ref{sec:pretrain} for results with a BC pretraining step added to all algorithms. 
%We observed that the RED reward is unable to distinguish between state-action pairs occuring frequently and those that are within support but occur infrequently. As a result, the imitator has no incentive to prefer frequent states which tend to have higher ground truth rewards. 
GAIL and ValueDICE perform comparably with each other, both outperforming behavioral cloning. We note that these results are somewhat unsurprising given that ValueDICE \cite{kostrikov2020valuedice} did not claim to improve demonstration efficiency over GAIL \cite{ho2016generative}, but rather focused on reducing the number of environment interactions. Both methods notably under-perform the expert on Ant-v3 and Humanoid-v3 which have the largest state-action spaces. Although minimizing the number of environment interactions was not a targeted goal of this work, we found that NDI roughly requires an order of magnitude less environment interactions than GAIL. Please see Appendix \ref{sec:environment_sample} for full environment sample complexity comparisons.

\subsection{Density Evaluation}
\label{sec:density}
In this section, we examine the learned density model $q_{\phi}$ for NDI+EBM and show that it highly correlates with the true mujoco rewards which are linear functions of forward velocity. 
% The ground truth reward of the considered Mujoco tasks is mainly proportional to the forward velocity of the robot measured at its Center of Mass (COM). Although COM velocity is not included in the state space of all of the considered environments, there does exist dimensions that measure the forward velocity from a different point on the robot, e.g at the hip joint. 
% Figure \ref{fig:density} shows that $q_{\phi}$ has a positive correlation with the forward velocity of the robots much like the true environment reward. Intuitively, a good density estimate should indeed have such correlations, since the true expert occupancy measure should positively correlate with forward velocity due to the expert attempting to consistently maintain high velocity.
We randomly sample test states $s$ and multiple test actions $a_{s}$ per test state, both from a uniform distribution with boundaries at the minimum/maximum state-action values in the demonstration set. We then visualize the log marginal $\log q_{\phi}(s) = \log \sum_{a_{s}} q_{\phi}(s, a_s)$ projected on to two state dimensions: one corresponding to the forward velocity of the robot and the other a random selection, e.g the knee joint angle. Each point in Figure \ref{fig:density} corresponds to a projection of a sampled test state $s$, and the colors scale with the value of $\log q_{\phi}(s)$. For all environments besides Humanoid, we found that the density estimate positively correlates with velocity even on uniformly drawn state-actions which were not contained in the demonstrations. We found similar correlations for Humanoid on states in the demonstration set. Intuitively, a good density estimate should indeed have such correlations, since the true expert occupancy measure should positively correlate with forward velocity due to the expert attempting to consistently maintain high velocity.

\begin{figure}[t]
    \centering
    \includegraphics[width=\textwidth]{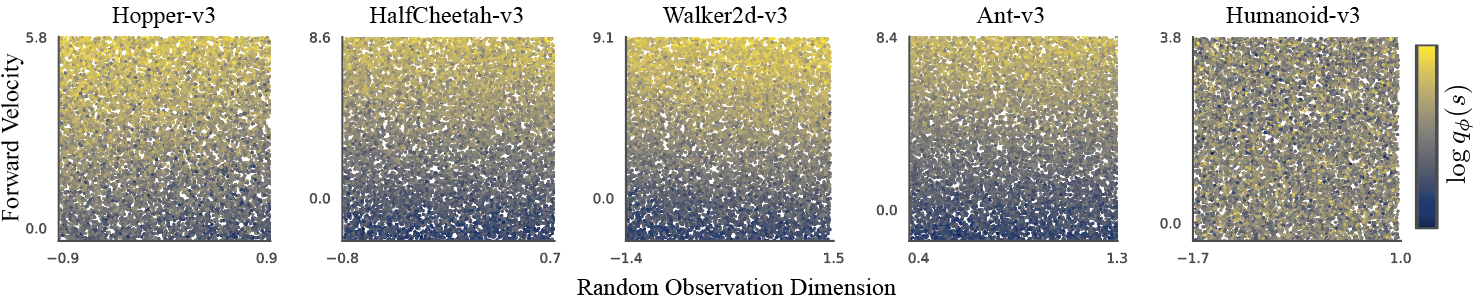}
    \vspace{-15pt}
    \caption{\textbf{Learned density visualization}. We randomly sample test states $s$ and multiple test actions $a_{s}$ per test state, both from a uniform distribution, then visualize the log marginal $\log q_{\phi}(s) = \log \sum_{a_{s}} q_{\phi}(s, a_s)$ projected onto two state dimensions: one corresponding to forward velocity and the other a random selection. Much like true reward function in Mujoco environments, we found that the log marginal positively correlates with forward velocity on 4/5 tasks.}
    \label{fig:density}
    \vspace{-10pt}
\end{figure}

\begin{table}[t]
    \setlength{\tabcolsep}{4pt}
    \renewcommand{\arraystretch}{1.}
    \small
    \caption{\textbf{Effect of varying MI reward weight $\lambda_f$} on (1). Task performance of NDI-EBM (top row) and (2). Imitation performance of NDI-EBM (bottom row) measured as the average KL divergence between $\pi, \pi_{E}$ on states $s$ sampled by running $\pi$ in the true environment, i.e $\E_{s \sim \pi}[D_{\mathrm{KL}}(\pi(\cdot | s) || \pi_{E}(\cdot | s))]$, normalized by the average $D_{\mathrm{KL}}$ between the random and expert policies. $D_{\mathrm{KL}}(\pi || \pi_{E})$ can be computed analytically since $\pi, \pi_{E}$ are conditional gaussians. Setting $\lambda_{f}$ too large hurts task performance while setting it too small is suboptimal for matching the expert occupancy. A middle point of $\lambda = 0.005$ achieves a balance between the two metrics.}
    \vspace{0pt}
    \label{table:lambda_f}
    \begin{center}
    \begin{sc}
    \begin{tabular}{lcccccc}
    \toprule
    & & Hopper & Half-Cheetah & Walker & Ant & Humanoid \\
    \midrule\midrule
    % $\lambda_{f} = 0$ & $3538.2 \pm 78.5$ & $5492. \pm 105.2$ & $5214.2 \pm 203.4$ & $4019.3 \pm 153.2$ & $5823.9 \pm 530.4$ \\
    \multirow{2}{*}{$\lambda_{f} = 0.0001$} & Reward & $\mathbf{3506 \pm 188}$ & $\mathbf{5697 \pm 805}$ & $\mathbf{5171 \pm 157}$ & $4158 \pm 523$ & $\mathbf{5752 \pm 632}$ \\
    & KL & $0.15 \pm 0.05$ & $0.32 \pm 0.15$ & $0.25 \pm 0.04$ & $0.51 \pm 0.05$ & $0.41 \pm 1.82$ \\
    \midrule
    \rowcolor{LightOrange} & Reward & $3458 \pm 210$ & $4511 \pm 569$ & $5061 \pm 135$ & $\mathbf{4293 \pm 431}$ & $5305 \pm 555$ \\
    \rowcolor{LightOrange} \multirow{-2}{*}{$\lambda_{f} = 0.005$} & KL & $\mathbf{0.11 \pm 0.02}$ & $\mathbf{0.17 \pm 0.09}$ & $\mathbf{0.22 \pm 0.14}$ & $\mathbf{0.32 \pm 0.12}$ & $\mathbf{0.12 \pm 0.14}$ \\
    \midrule
    \multirow{2}{*}{$\lambda_{f} = 0.1$} & Reward & $1057 \pm 29$ & $103 \pm 59$ & $2710 \pm 501$ & $-1021 \pm 21$ & $142 \pm 50$ \\
    & KL & $0.78 \pm 0.13$ & $1.41 \pm 0.51$ & $0.41 \pm 0.11$ & $2.41 \pm 1.41$ & $0.89 \pm 0.21$ \\
    \midrule\midrule
    Expert & Reward & $3567 \pm 4$ & $4142 \pm 132$ & $5006 \pm 472$ & $4362 \pm 827$ & $5417 \pm 2286$ \\
    \bottomrule
    \end{tabular}
    \end{sc}
    \end{center}
\vspace{-15pt}
\end{table}

\subsection{Ablation studies}
\label{sec:ablation}
Table \ref{table:lambda_f} shows the effect of the varying $\lambda_f$ on task (reward) and imitation performance (KL), i.e similarities between $\pi, \pi_{E}$ measured as $\E_{s \sim \pi}[D_{\mathrm{KL}}(\pi(\cdot | s) || \pi_{E}(\cdot | s))]$. Setting $\lambda_{f}$ too large ($\geq 0.1$) hurts both task and imitation performance as the MI reward $r_f$ dominates the RL objective leading to excessive state-action entropy maximization. Setting it too small ($\leq 0.0001$), i.e only maximizing policy entropy $\mc{H}(\pi_{\theta})$, turns out to benefit task performance, sometimes enabling the imitator to outperform the expert by concentrating most of it's trajectory probability mass to the mode of the expert's trajectory distribution. However, the boosted task performance comes at the cost of suboptimal imitation performance, e.g imitator cheetah running faster than the expert. We found that a middle point of $\lambda_f = 0.005$ simultaneously achieves expert level task performance and good imitation performance. This phenomenon is analogous to how controlling the strength of entropy regularization on the generator can lead to either mode-seeking vs mode covering behavior in GANs. In summary, these results show that state-action entropy $\mc{H}^{f}$ maximization improves distribution matching between $\pi, \pi_{E}$ over policy entropy $\mc{H}(\pi_{\theta})$ maximization, but distribution matching may not be ideal for task performance maximization, e.g in apprenticeship learning settings. 

% First, we found that $r_{f}$ doesn't lead to a statistically significant difference in imitation performance, which is due to the fact that the learned density already mirrors the key aspect of the ground truth reward function. Second, we found that spectral normalization is key to preventing the density model from overfitting, and hence strongly affects imitation performance, especially since we work in the extremely low data regime. 

% We found that an appropriate value of $\lambda_f$ finds a balance between task performance and imitating the expert. \kh{important to find a compelling explanation of why task performance is high when $\lambda_f = 0$}

\section{Discussion and Outlook}
This work’s main contribution is a new principled framework for IL and an algorithm that obtains state-of-the-art demonstration efficiency. One future direction is to apply NDI to harder visual IL tasks for which AIL is known perform poorly. While the focus of this work is to improve on demonstration efficiency, another important IL performance metric is environment sample complexity. Future works could explore combining off-policy RL or model-based RL with NDI to improve on this end. Finally, there is a rich space of questions to answer regarding the effectiveness of the SAELBO reward $r_{f}$. We posit that, for example, in video game environments $r_{f}$ may be crucial for success since the expert spends most of its time in levels that are extremely hard to reach with just action entropy maximization. Furthermore, one could improve on the tightness of SAELBO by incorporating negative samples \cite{oord2018cpc} and learning the critic function $f$ so that it is close to the optimal critic. 

% \section*{Broader Impact}
% Efficient and robust imitation learning methods could be essential for solving wide range of real-world decision-making problems. This work makes both practical and theoretical contributions. The theoretical contributions, which include new entropy bounds and model-free gradient estimators, are unlikely to have a direct impact on society. The practical algorithm that is built on these theoretical insights is however widely applicable. Imitation learning is a key building block towards improving the capability of autonomous and robotic agents. Being able to imitate complex behaviors from a small number of demonstrations is likely going to increase the capabilities of such systems. As a result, the potential impact on society is large. Better imitation learning algorithms could lead to job losses and unemployment, but also save lives by replacing humans performing dangerous but necessary jobs. 

\newpage
\bibliography{main}

\begin{thebibliography}{59}
\providecommand{\natexlab}[1]{#1}
\providecommand{\url}[1]{\texttt{#1}}
\expandafter\ifx\csname urlstyle\endcsname\relax
  \providecommand{\doi}[1]{doi: #1}\else
  \providecommand{\doi}{doi: \begingroup \urlstyle{rm}\Url}\fi

\bibitem[Belghazi et~al.(2018)Belghazi, Baratin, Rajeswar, Ozair, Bengio,
  Courville, and Hjelm]{belghazi2018entropy}
Mohamed~Ishmael Belghazi, Aristide Baratin, Sai Rajeswar, Sherjil Ozair, Yoshua
  Bengio, Aaron Courville, and R~Devon Hjelm.
\newblock Mine: Mutual information neural estimation.
\newblock \emph{arXiv preprint arXiv:1801.04062}, 2018.

\bibitem[Brantley et~al.(2020)Brantley, Sun, and
  Henaff]{brantley2020disagreement}
Kiante Brantley, Wen Sun, and Mikael Henaff.
\newblock Disagreement-regularized imitation learning.
\newblock 2020.

\bibitem[Brockman et~al.(2016)Brockman, Cheung, Pettersson, Schneider,
  Schulman, Tang, and Wojciech]{brockman2017openaigym}
Greg Brockman, Vicki Cheung, Ludwig Pettersson, Jonas Schneider, John Schulman,
  Jie Tang, and Zaremba Wojciech.
\newblock Openai gym.
\newblock \emph{arXiv preprint arXiv:1606.01540}, 2016.

\bibitem[Brown(1951)]{brown1951fictitious}
George Brown.
\newblock Iterative solution of games by fictitious play.
\newblock \emph{Activity Analysis of Production and Allocation}, 1951.

\bibitem[Burda et~al.(2018)Burda, Edwards, Storkey, and Klimov]{burda2018rnd}
Yuri Burda, Harrison Edwards, Amos Storkey, and Oleg Klimov.
\newblock Exploration by random network distillation.
\newblock \emph{arXiv preprint arXiv:1810.12894}, 2018.

\bibitem[Dieng et~al.(2019)Dieng, Ruiz, Blei, and Titsias]{dieng2019prescribed}
Adji Dieng, Francisco Ruiz, David~M. Blei, and Michalis~K. Titsias.
\newblock Prescribed generative adversarial networks.
\newblock \emph{arXiv preprint arXiv:1910.04302}, 2019.

\bibitem[Dinh et~al.(2014)Dinh, Krueger, and Bengio]{dinh2014nice}
L~Dinh, D~Krueger, and Y~Bengio.
\newblock {NICE}: Non-linear independent components estimation.
\newblock \emph{arXiv preprint arXiv:1410.8516}, 2014.

\bibitem[Dinh et~al.(2016)Dinh, Sohl-Dickstein, and Bengio]{dinh2016density}
Laurent Dinh, Jascha Sohl-Dickstein, and Samy Bengio.
\newblock Density estimation using real {NVP}.
\newblock \emph{arXiv preprint arXiv:1605.08803}, May 2016.

\bibitem[Du and Mordatch(2018)]{du2018ebm}
Yilun Du and Igor Mordatch.
\newblock Implicit generation and generalization in energy-based models.
\newblock \emph{arXiv preprint arXiv:1903.08689}, 2018.

\bibitem[Du and Mordatch(2019)]{du2019implicit}
Yilun Du and Igor Mordatch.
\newblock Implicit generation and generalization in energy-based models.
\newblock \emph{arXiv preprint arXiv:1903.08689}, 2019.

\bibitem[Fu et~al.(2017)Fu, Luo, and Levine]{fu2017learning}
Justin Fu, Katie Luo, and Sergey Levine.
\newblock Learning robust rewards with adversarial inverse reinforcement
  learning.
\newblock \emph{arXiv preprint arXiv:1710.11248}, 2017.

\bibitem[Gao et~al.(2019)Gao, Nijkamp, Kingma, Xu, Dai, and Wu]{gao2019flow}
Ruiqi Gao, Erik Nijkamp, Diederik~P Kingma, Zhen Xu, Andrew~M Dai, and
  Ying~Nian Wu.
\newblock Flow contrastive estimation of energy-based models.
\newblock \emph{arXiv preprint arXiv:1912.00589}, 2019.

\bibitem[Germain et~al.(2015)Germain, Gregor, Murray, and
  Larochelle]{germain2015made}
Mathieu Germain, Karol Gregor, Iain Murray, and Hugo Larochelle.
\newblock Made: Masked autoencoder for distribution estimation.
\newblock \emph{arXiv preprint arXiv:1502.03509}, 2015.

\bibitem[Goodfellow et~al.(2014)Goodfellow, Pouget-Abadie, Mirza, Xu,
  Warde-Farley, Ozair, Courville, and Bengio]{goodfellow2014generative}
Ian Goodfellow, Jean Pouget-Abadie, Mehdi Mirza, Bing Xu, David Warde-Farley,
  Sherjil Ozair, Aaron Courville, and Yoshua Bengio.
\newblock Generative adversarial nets.
\newblock In \emph{Advances in neural information processing systems}, pages
  2672--2680, 2014.

\bibitem[Guo et~al.(2014)Guo, Singh, Lee, Lewis, and Wang]{guo2014ataridagger}
Xiaoxiao Guo, Satinder Singh, Honglak Lee, Richard~L. Lewis, and Xiaoshi Wang.
\newblock Deep learning for real-time atari game play using offline monte-carlo
  tree search planning.
\newblock \emph{Advances in Neural Information Processing Systems 27 (NIPS
  2014)}, 2014.

\bibitem[Gutmann and Hyv{\"a}rinen(2010)]{gutmann2010noise}
Michael Gutmann and Aapo Hyv{\"a}rinen.
\newblock Noise-contrastive estimation: A new estimation principle for
  unnormalized statistical models.
\newblock In \emph{Proceedings of the Thirteenth International Conference on
  Artificial Intelligence and Statistics}, pages 297--304, 2010.

\bibitem[Haarnoja et~al.(2017)Haarnoja, Tang, Abbeel, and
  Levine]{haarnoja2017reinforcement}
Tuomas Haarnoja, Haoran Tang, Pieter Abbeel, and Sergey Levine.
\newblock Reinforcement learning with deep energy-based policies.
\newblock \emph{arXiv preprint arXiv:1702.08165}, 2017.

\bibitem[Haarnoja et~al.(2018)Haarnoja, Zhou, Abbeel, and
  Levine]{haarnoja2018soft}
Tuomas Haarnoja, Aurick Zhou, Pieter Abbeel, and Sergey Levine.
\newblock Soft {Actor-Critic}: {Off-Policy} maximum entropy deep reinforcement
  learning with a stochastic actor.
\newblock \emph{arXiv preprint arXiv:1801.01290}, January 2018.

\bibitem[Hazan et~al.(2019)Hazan, Kakade, Singh, and Soest]{hazan2019maxent}
Elad Hazan, Sham Kakade, Karan Singh, and Abby~Van Soest.
\newblock Provably efficient maximum entropy exploration.
\newblock \emph{arXiv preprint arXiv:1812.02690}, 2019.

\bibitem[Hill et~al.(2018)Hill, Raffin, Ernestus, Gleave, Kanervisto, Traore,
  Dhariwal, Hesse, Klimov, Nichol, Plappert, Radford, Schulman, Sidor, and
  Wu]{stable-baselines}
Ashley Hill, Antonin Raffin, Maximilian Ernestus, Adam Gleave, Anssi
  Kanervisto, Rene Traore, Prafulla Dhariwal, Christopher Hesse, Oleg Klimov,
  Alex Nichol, Matthias Plappert, Alec Radford, John Schulman, Szymon Sidor,
  and Yuhuai Wu.
\newblock Stable baselines.
\newblock \url{https://github.com/hill-a/stable-baselines}, 2018.

\bibitem[Ho and Ermon(2016)]{ho2016generative}
Jonathan Ho and Stefano Ermon.
\newblock Generative adversarial imitation learning.
\newblock In \emph{Advances in Neural Information Processing Systems}, pages
  4565--4573, 2016.

\bibitem[Huszár(2017)]{huszar2017implicit}
Ferenc Huszár.
\newblock Variational inference using implicit distributions.
\newblock \emph{arXiv preprint arXiv:1702.08235}, 2017.

\bibitem[Hyv{\"a}rinen(2005)]{hyvarinen2005estimation}
Aapo Hyv{\"a}rinen.
\newblock Estimation of {Non-Normalized} statistical models by score matching.
\newblock \emph{Journal of machine learning research: JMLR}, 6\penalty0
  (Apr):\penalty0 695--709, 2005.
\newblock ISSN 1532-4435, 1533-7928.

\bibitem[Islam et~al.(2019)Islam, Seraj, Bacon, and Precup]{islam2019maxent}
Riashat Islam, Raihan Seraj, Pierre-Luc Bacon, and Doina Precup.
\newblock Entropy regularization with discounted future state distribution in
  policy gradient methods.
\newblock \emph{arXiv preprint arXiv:1912.05104}, 2019.

\bibitem[Jin et~al.(2019)Jin, Netrapalli, and Jordan]{jin2019minmax}
Chi Jin, Praneeth Netrapalli, and Michael~I. Jordan.
\newblock What is local optimality in nonconvex-nonconcave minimax
  optimization?
\newblock \emph{arXiv preprint arXiv:1902.00618}, 2019.

\bibitem[Ke et~al.(2020)Ke, Barnes, Sun, Lee, Choudhury, and
  Srinivasa]{ke2019divergenceimitation}
Liyiming Ke, Matt Barnes, Wen Sun, Gilwoo Lee, Sanjiban Choudhury, and
  Siddhartha Srinivasa.
\newblock Imitation learning as f-divergence minimization.
\newblock \emph{arXiv preprint arXiv:1905.12888}, 2020.

\bibitem[Kim and Park(2018)]{kim2018mmdgail}
Kee-Eung Kim and Hyun~Soo Park.
\newblock Imitation learning via kernel mean embedding.
\newblock \emph{AAAI}, 2018.

\bibitem[Kim et~al.(2019)Kim, Gu, Song, Zhao, and Ermon]{kim2019cross}
Kuno Kim, Yihong Gu, Jiaming Song, Shengjia Zhao, and Stefano Ermon.
\newblock Domain adaptive imitation learning.
\newblock \emph{arXiv preprint arXiv:1910.00105}, 2019.

\bibitem[Kingma and Ba(2014)]{kingma2014adam}
Diederik~P Kingma and Jimmy Ba.
\newblock Adam: A method for stochastic optimization.
\newblock \emph{arXiv preprint arXiv:1412.6980}, December 2014.

\bibitem[Kingma and Dhariwal(2018)]{kingma2018glow}
Durk~P Kingma and Prafulla Dhariwal.
\newblock Glow: Generative flow with invertible 1x1 convolutions.
\newblock In \emph{Advances in Neural Information Processing Systems}, pages
  10215--10224, 2018.

\bibitem[Kostrikov et~al.(2020)Kostrikov, Nachum, and
  Tompson]{kostrikov2020valuedice}
Ilya Kostrikov, Ofir Nachum, and Jonathan Tompson.
\newblock Imitation learning via off-policy distribution matching.
\newblock 2020.

\bibitem[Lee et~al.(2019)Lee, Eysenbach, Parisotto, Xing, Levine, and
  Salakhutdinov]{lee2019marginal}
Lisa Lee, Benjamin Eysenbach, Emilio Parisotto, Eric Xing, Sergey Levine, and
  Ruslan Salakhutdinov.
\newblock Efficient exploration via state marginal matching.
\newblock \emph{arXiv preprint arXiv:1906.05274}, 2019.

\bibitem[Liu et~al.(2007)Liu, Lafferty, and Wasserman]{liu2007nonparametric}
Han Liu, John Lafferty, and Larry Wasserman.
\newblock Sparse nonparametric density estimation in high dimensions using the
  rodeo.
\newblock volume~2 of \emph{Proceedings of Machine Learning Research}, pages
  283--290, San Juan, Puerto Rico, 21--24 Mar 2007. PMLR.

\bibitem[Miyato et~al.(2018)Miyato, Kataoka, Koyama, and
  Yoshida]{miyato2018spectral}
Takeru Miyato, Toshiki Kataoka, Masanori Koyama, and Yuichi Yoshida.
\newblock Spectral normalization for generative adversarial networks.
\newblock \emph{arXiv preprint arXiv:1802.05957}, 2018.

\bibitem[Mohamed and Lakshminarayanan(2016)]{mohamed2016implicit}
Shakir Mohamed and Balaji Lakshminarayanan.
\newblock Learning in implicit generative models.
\newblock \emph{arXiv preprint arXiv:1610.03483}, 2016.

\bibitem[Nguyen et~al.(2010)Nguyen, Wainwright, and Jordan]{nguyen2010nwj}
XuanLong Nguyen, Martin~J. Wainwright, and Michael~I. Jordan.
\newblock Estimating divergence functionals and the likelihood ratio by convex
  risk minimization.
\newblock \emph{arXiv preprint arXiv:0809.0853}, 2010.

\bibitem[Nowozin et~al.(2016)Nowozin, Cseke, and Tomioka]{nowozin2016f}
Sebastian Nowozin, Botond Cseke, and Ryota Tomioka.
\newblock {{f-GAN}}: Training generative neural samplers using variational
  divergence minimization.
\newblock \emph{arXiv preprint arXiv:1606.00709}, June 2016.

\bibitem[Papamakarios et~al.(2017)Papamakarios, Pavlakou, and
  Murray]{papamakarios2017maf}
George Papamakarios, Theo Pavlakou, and Iain Murray.
\newblock Masked autoregressive flow for density estimation.
\newblock \emph{arXiv preprint arXiv:1705.07057}, 2017.

\bibitem[Pathak et~al.(2017)Pathak, Agrawal, Efros, and
  Darrell]{pathak2017inverse}
Deepak Pathak, Pulkit Agrawal, Alexei~A Efros, and Trevor Darrell.
\newblock Curiosity-driven exploration by self-supervised prediction.
\newblock \emph{arXiv preprint arXiv:1705.05363}, 2017.

\bibitem[Pomerleau(1991)]{pomerleau1991efficient}
Dean~A Pomerleau.
\newblock Efficient training of artificial neural networks for autonomous
  navigation.
\newblock \emph{Neural computation}, 3\penalty0 (1):\penalty0 88--97, 1991.
\newblock ISSN 0899-7667.

\bibitem[Reddy et~al.(2017)Reddy, Dragan, and Levine]{reddy2017sqil}
Siddharth Reddy, Anca~D. Dragan, and Sergey Levine.
\newblock Sqil: Imitation learning via reinforcement learning with sparse
  rewards.
\newblock \emph{arXiv preprint arXiv:1905.11108}, 2017.

\bibitem[Ross and Bagnell(2010)]{ross2010efficient}
St{\'e}phane Ross and Drew Bagnell.
\newblock Efficient reductions for imitation learning.
\newblock In \emph{Proceedings of the thirteenth international conference on
  artificial intelligence and statistics}, pages 661--668, 2010.

\bibitem[Ross et~al.(2011{\natexlab{a}})Ross, Gordon, and Bagnell]{ross2011a}
St{\'e}phane Ross, Geoffrey Gordon, and Drew Bagnell.
\newblock A reduction of imitation learning and structured prediction to
  no-regret online learning.
\newblock In \emph{Proceedings of the fourteenth international conference on
  artificial intelligence and statistics}, pages 627--635, 2011{\natexlab{a}}.

\bibitem[Ross et~al.(2011{\natexlab{b}})Ross, Gordon, and
  Bagnell]{ross2011reduction}
St{\'e}phane Ross, Geoffrey Gordon, and Drew Bagnell.
\newblock A reduction of imitation learning and structured prediction to
  no-regret online learning.
\newblock In \emph{Proceedings of the fourteenth international conference on
  artificial intelligence and statistics}, pages 627--635, 2011{\natexlab{b}}.

\bibitem[Ross et~al.(2013)Ross, Melik-Barkhudarov, Shankar, Wendel, Dey,
  Bagnell, and Hebert]{ross2013uav}
Stephane Ross, Narek Melik-Barkhudarov, Kumar~Shaurya Shankar, Andreas Wendel,
  Debadeepta Dey, Andrew~J. Bagnell, and Martial Hebert.
\newblock Learning monocular reactive uav control in cluttered natural
  environments.
\newblock \emph{International Conference on Robotics and Automation (ICRA)},
  2013.

\bibitem[Salimans et~al.(2016)Salimans, Goodfellow, Zaremba, Cheung, Radford,
  and Chen]{salimans2016gantechniques}
Tim Salimans, Ian Goodfellow, Wojciech Zaremba, Vicki Cheung, Alec Radford, and
  Xi~Chen.
\newblock Improved techniques for training gans.
\newblock \emph{arXiv:1606.03498}, 2016.

\bibitem[Song et~al.(2018)Song, Ren, Sadigh, and Ermon]{song2018multi}
Jiaming Song, Hongyu Ren, Dorsa Sadigh, and Stefano Ermon.
\newblock Multi-agent generative adversarial imitation learning.
\newblock 2018.

\bibitem[Song and Ermon(2019)]{song2019gradients}
Yang Song and Stefano Ermon.
\newblock Generative modeling by estimating gradients of the data distribution.
\newblock \emph{arXiv preprint arXiv:1907.05600}, 2019.

\bibitem[Song et~al.(2019)Song, Garg, Shi, and Ermon]{song2019sliced}
Yang Song, Sahaj Garg, Jiaxin Shi, and Stefano Ermon.
\newblock Sliced score matching: A scalable approach to density and score
  estimation.
\newblock \emph{arXiv preprint arXiv:1905.07088}, 2019.

\bibitem[Todorov et~al.(2012)Todorov, Erez, and Tassa]{todorov2012mujoco}
E~Todorov, T~Erez, and Y~Tassa.
\newblock {MuJoCo}: A physics engine for model-based control.
\newblock In \emph{2012 {IEEE/RSJ} International Conference on Intelligent
  Robots and Systems}, pages 5026--5033, October 2012.
\newblock \doi{10.1109/IROS.2012.6386109}.

\bibitem[Todorov(2014)]{todorov2014invertible}
Emanuel Todorov.
\newblock Convex and analytically-invertible dynamics with contacts and
  constraints: Theory and implementation in mujoco.
\newblock \emph{IEEE International Conference on Robotics and Automation
  (ICRA)}, 2014.

\bibitem[Uria et~al.(2013)Uria, Murray, and Larochelle]{uria2013rnade}
Benigno Uria, Iain Murray, and Hugo Larochelle.
\newblock {{RNADE}}: The real-valued neural autoregressive density-estimator.
\newblock In C~J~C Burges, L~Bottou, M~Welling, Z~Ghahramani, and K~Q
  Weinberger, editors, \emph{Advances in Neural Information Processing Systems
  26}, pages 2175--2183. Curran Associates, Inc., 2013.

\bibitem[Uria et~al.(2016)Uria, C{\^o}t{\'e}, Gregor, Murray, and
  Larochelle]{uria2016neural}
Benigno Uria, Marc-Alexandre C{\^o}t{\'e}, Karol Gregor, Iain Murray, and Hugo
  Larochelle.
\newblock Neural autoregressive distribution estimation.
\newblock \emph{The Journal of Machine Learning Research}, 17\penalty0
  (1):\penalty0 7184--7220, 2016.

\bibitem[van~den Oord et~al.(2016)van~den Oord, Kalchbrenner, and
  Kavukcuoglu]{oord2016pixel}
Aaron van~den Oord, Nal Kalchbrenner, and Koray Kavukcuoglu.
\newblock Pixel recurrent neural networks.
\newblock \emph{arXiv preprint arXiv:1601.06759}, January 2016.

\bibitem[Van Den~Oord et~al.(2018)Van Den~Oord, Li, and Vinyals]{oord2018cpc}
Aaron Van Den~Oord, Yazhe Li, and Oriol Vinyals.
\newblock Representation learning with contrastive predictive coding.
\newblock \emph{arXiv preprint arXiv:1807.03748}, 2018.

\bibitem[Wang et~al.(2019)Wang, Ciliberto, Amadori, and Demiris]{wang2020red}
Ruohan Wang, Carlo Ciliberto, Pierluigi Amadori, and Yiannis Demiris.
\newblock Random expert distillation: Imitation learning via expert policy
  support estimation.
\newblock 2019.

\bibitem[Wang et~al.(2017)Wang, Merel, Reed, Wayne, Freitas, and
  Heess]{wang2017diverse}
Ziyu Wang, Josh Merel, Scott Reed, Greg Wayne, Nando~de Freitas, and Nicolas
  Heess.
\newblock Robust imitation of diverse behaviors.
\newblock \emph{arXiv preprint arXiv:1707.02747}, 2017.

\bibitem[Wu et~al.(2019)Wu, Piergiovanni, and Ryoo]{wu2019futurebc}
Alan Wu, AJ~Piergiovanni, and Michael~S. Ryoo.
\newblock Model-based behavioral cloning with future image similarity learning.
\newblock \emph{arXiv preprint arXiv:1910.03157}, 2019.

\bibitem[Yu et~al.(2020)Yu, Song, Song, and Ermon]{yu2020training}
Lantao Yu, Yang Song, Jiaming Song, and Stefano Ermon.
\newblock Training deep energy-based models with f-divergence minimization.
\newblock \emph{arXiv preprint arXiv:2003.03463}, 2020.

\end{thebibliography}

\newpage
\appendix
\begin{center}
\Large \textbf{Imitation with Neural Density Models - Appendix}
\end{center}

\section{Proofs}
\label{sec:proofs}

We first make clear the assumptions on the MDPs considered henceforth. 

\begin{restatable}{assumption}{assumption_mdp} 
\label{ass:mdp}
All considered MDPs have deterministic dynamics governed by a transition function $P: \S \times \A \rightarrow \S$. Furthermore, $P$ is injective with respect to $a \in \A$, i.e $\forall s, a, a'$ it holds that $a \neq a' \Rightarrow P(s, a) \neq P(s, a')$. 
\end{restatable}

We note that Assumption \ref{ass:mdp} holds for most continuous robotics and physics environments as they are deterministic and inverse dynamics functions $P^{-1}: \S \times \S \rightarrow \A$ have been successfully used in benchmark RL environments such as Mujoco \cite{todorov2012mujoco, todorov2014invertible}  and Atari \cite{pathak2017inverse}. 

\subsection{Proof of Theorem \ref{thm:entropy_lowerbound}}
\label{sec:proof_thm_entropy_lowerbound}
We first prove some useful lemmas. 

\begin{lemma}
\label{lem:extended_entropy}
Let $\gH: \hat{p} \mapsto \int_{\gX} \hat{p}(x) \log \hat{p}(x)$ denote the generalized entropy defined on the extended domain of non-normalized densities $\Delta^{+} = \{\hat{p}: \gX \rightarrow \RR^{+}$ | $\exists Z > 0$ st $\int_{\gX} \hat{p}(x)/Z = 1\}$ where $\RR^{+}$ is the set of non-negative real numbers. $\gH$ is concave. 
\end{lemma}
\begin{proof}
\begin{align*}
\gH(\lambda \hat{p} + (1 - \lambda) \hat{q}) 
&= -\sum_x \big(\lambda \hat{p}(x) + (1 - \lambda) \hat{q}(x)\big) \log \big(\lambda \hat{p}(x) + (1 - \lambda) \hat{q}(x)\big) \\
&\geq -\sum_x \lambda \hat{p}(x) \log \hat{p}(x) + (1 - \lambda) \hat{q}(x) \log \hat{q}(x) \\
&= -\lambda \sum_x \hat{p}(x) \log \hat{p}(x) + (1 - \lambda) \sum_x \hat{q}(x) \log \hat{q}(x) \\
&= \lambda \gH(\hat{p}) + (1 - \lambda) \gH(\hat{q})
\end{align*}
where the inequality in the second line holds since $\forall x$, it holds that $\hat{p}(x), \hat{q}(x) \in \RR^+$, and the map $f(u) := -u\log u$ is strictly concave on $u \in \RR^+$; this follows from the fact that $f'(u) = -(1 + \log u)$ is strictly decreasing on $\RR^+$. 
\end{proof}

\begin{lemma}
\label{lem:conditional}
Let MDP $\gM$ satisfy Assumption \ref{ass:mdp}. Let $\{s_{t}, a_{t}\}_{t=0}^{\infty}$ be the stochastic process realized by sampling an initial state from $s_{0} \sim P_{0}(s)$ then running policy $\pi$ with determinsitic, injective dynamics function $P$, i.e $a_{t} \sim \pi(\cdot | s_{t}), s_{t+1} = P(s_{t}, a_{t})$. Then $\forall t \geq 1$,
\begin{align*}
    \gH(s_{t} | s_{t-1}) = \gH(a_{t-1} | s_{t-1})
\end{align*}
\end{lemma}
\begin{proof}
We expand $\gH(s_t,a_{t-1}|s_{t-1})$ in two different ways:
\begin{align*}
\gH(s_t,a_{t-1}|s_{t-1}) &= \gH(s_t|s_{t-1}, a_{t-1}) + \gH(a_{t-1}|s_{t-1}) = 0 + \gH(a_{t-1}|s_{t-1}) \\
\gH(s_t,a_{t-1}|s_{t-1}) &= \gH(a_{t-1}|s_{t-1}, s_t) + \gH(s_t|s_{t-1}) = 0 + \gH(s_t|s_{t-1})
\end{align*}
The $\gH(s_t|s_{t-1}, a_{t-1}), \gH(a_{t-1}|s_{t-1}, s_t)$ terms can be zero'd out due to the determinstic, injective dynamics assumption. Thus, we conclude that $\gH(s_t|s_{t-1}) = \gH(a_{t-1}|s_{t-1})$.  
\end{proof}

\newpage

\lowerbound*

\begin{proof}

\begin{align*}
\gH(\rho_{\pi_{\theta}}(s, a)) 
&= -\sum_{s, a} \rho_{\pi_{\theta}}(s, a) \log \rho_{\pi_{\theta}}(s, a) \\
&= -\sum_{s, a} \rho_{\pi_{\theta}}(s, a) \log \frac{\rho_{\pi_{\theta}}(s, a)}{\rho_{\pi_{\theta}}(s)} 
- \sum_{s, a} \rho_{\pi_{\theta}}(s, a) \log \rho_{\pi_{\theta}}(s) \\
&= -\sum_{s, a} \rho_{\pi_{\theta}}(s, a) \log \pi_{\theta}(a | s) - \sum_{s} \rho_{\pi_{\theta}}(s) \log \rho_{\pi_{\theta}}(s) \\
&= \gH(\pi_{\theta}) + \gH(\rho_{\pi_{\theta}}(s)) 
\end{align*}
We now lower bound the state-marginal occupancy entropy term 
\begin{align*}
\gH(\rho_{\pi_{\theta}}(s)) 
&= \gH(\sum_{t = 0}^{\infty} \gamma^{t} p_{\theta, t}(s)) \\
&\geq \sum_{t = 0}^{\infty} \gamma^{t} \gH(s_{t}) && \text{Lemma  \ref{lem:extended_entropy}} \numberthis \label{eq:jensen} \\
&= \Big(\gH(s_{0}) + \sum_{t = 1}^{\infty} \gamma^{t} \gH(s_{t})\Big)\\
&= \gH(s_{0}) 
+ \sum_{t = 1}^{\infty} \gamma^{t} \gH(s_{t} | s_{t-1}) 
+ \sum_{t = 1}^{\infty} \gamma^{t} I(s_{t}; s_{t - 1})\\
\end{align*}
Let us consider each term separately starting with the entropy term: 
\begin{align*}
\sum_{t = 1}^{\infty} \gamma^{t} \gH(s_{t} | s_{t-1})
&= \sum_{t = 1}^{\infty} \gamma^{t} \gH(a_{t-1} | s_{t-1}) && \text{Lemma \ref{lem:conditional}}\\
&= \gamma \sum_{t = 0}^{\infty} \gamma^{t} \gH(a_{t} | s_{t}) \\
&= \gamma \sum_{t = 0}^{\infty} \gamma^{t} \E_{p(s_{t}, a_{t})}[-\log p(a_{t} | s_{t})] \\
&= \gamma \sum_{t = 0}^{\infty} \gamma^{t} \E_{\pi_{\theta}}[-\log p(a_{t} | s_{t})] \\
&= \gamma \E_{\pi_{\theta}}[-\sum_{t = 0}^{\infty} \gamma^{t} \log \pi_{\theta}(a_{t} | s_{t})] \\
&= \gamma \gH(\pi_{\theta}) 
\end{align*}

We now lower bound the Mutual Information (MI) term using the bound of Nguyen, Wainright, and Jordan \cite{nguyen2010nwj}, also known as the $f$-GAN KL \cite{nowozin2016f} and MINE-$f$ \cite{belghazi2018entropy}. For random variables $X, Y$ distributed according to $p_{\theta_{xy}}(x, y), p_{\theta_{x}}(x), p_{\theta_{y}}(y)$ where $\theta = (\theta_{xy}, \theta_{x}, \theta_{y})$, and any critic function $f(x, y)$, it holds that $I(X, Y | \theta) \geq I^{f}_{\mathrm{NWJ}}(X; Y | \theta)$ where, 
\begin{equation}
I^f_{\mathrm{NWJ}}(X; Y) := \E_{p_{\theta_{xy}}}[f(x, y)] - e^{-1}\E_{p_{\theta_{x}}}[\E_{p_{\theta_{y}}}[e^{f(x, y)}]]
\label{eq:inwj_sup}
\end{equation}
This bound is tight when $f$ is chosen to be the optimal critic $f^*(x, y) = \log \frac{p_{\theta_{xy}}(x, y)}{p_{\theta_{x}}(x)p_{\theta_y}(y)} + 1$. Applying this bound we obtain: 

\begin{align*}
\sum_{t = 1}^{\infty} \gamma^{t} I(s_{t}; s_{t - 1} | \theta) 
&\geq \sum_{t = 1}^{\infty} \gamma^{t} I_{\mathrm{NWJ}}^{f}(s_{t}; s_{t - 1} | \theta) \\ 
&= \gamma \sum_{t = 0}^{\infty} \gamma^{t} I_{\mathrm{NWJ}}^{f}(s_{t+1}; s_{t} | \theta)
\end{align*}

Combining all the above results, 
\begin{align*}
\gH(\rho_{\pi_{\theta}}(s, a)) 
&= \gH(\pi_{\theta}) + \gH(\rho_{\pi_{\theta}}(s)) \\
&= \gH(\pi_{\theta}) 
+\gH(s_{0}) 
+ \sum_{t = 1}^{\infty} \gamma^{t} \gH(s_{t} | s_{t-1}) 
+ \sum_{t = 1}^{\infty} \gamma^{t} I(s_{t}; s_{t - 1} | \theta) \\
&\geq \gH(s_{0}) 
+ (1 + \gamma)\gH(\pi_{\theta}) 
+ \gamma \sum_{t = 0}^{\infty} \gamma^{t} I_{\mathrm{NWJ}}^{f}(s_{t+1}; s_{t} | \theta) 
\end{align*}
Setting $\gH^{f}(\rho_{\pi_{\theta}}) := \gH(s_{0}) 
+ (1 + \gamma)\gH(\pi_{\theta}) 
+ \gamma \sum_{t = 0}^{\infty} \gamma^{t} I_{\mathrm{NWJ}}^{f}(s_{t+1}; s_{t} | \theta)$ concludes the proof. 
\end{proof}

\textbf{Tightness of the SAELBO}: We call $\gH^{f}(\rho_{\pi_{\theta}})$ the State-Action Entropy Lower Bound (SAELBO). There are two potential sources of slack for the SAELBO. The first source is from the application of Jensen's inequality in Eq. \ref{eq:jensen}. This slack becomes smaller as $p_{\theta, t}$ converges to a stationary distribution as $t \rightarrow \infty$. The second source is the MI lowerbound $I_{\mathrm{NWJ}}$ in Eq. \ref{eq:inwj_sup} , which can be made tight if $f$ is sufficiently flexible and chosen (or learned) to be the optimal critic. 

\newpage

\subsection{Proof of Theorem \ref{thm:gradient}}
\label{sec:proof_thm_gradient}

\entropygradient*

\begin{proof}
We take the gradient of the SAELBO $\gH^{f}(\rho_{\pi_{\theta}})$ w.r.t $\theta$
\begin{align*}
\grad \gH^{f}(\rho_{\pi_{\theta}}) = \grad \gH(s_{0}) 
+ \grad (1 + \gamma)\gH(\pi_{\theta}) 
+ \grad \gamma \sum_{t = 0}^{\infty} \gamma^{t} I_{\mathrm{NWJ}}^{f}(s_{t+1}; s_{t} | \theta) 
\end{align*}
The first term vanishes, so we can consider the second and third term separately. Using the standard MaxEntRL policy gradient result (e.g Lemma A.1 of \cite{ho2016generative}), 
\begin{align*}
\grad (1 + \gamma)\gH(\pi_{\theta}) 
&= \grad \E_{\pi_{\theta}}[-\sum_{t = 0}^{\infty} \gamma^t (1 + \gamma) \log q_{\pi}(a_{t} | s_{t})] \\
&= \grad J(\pi_{\theta}, \bar{r}=r_{\pi}) \numberthis \label{eq:ent_gradient}
\end{align*}
Now for the third term, we further expand the inner terms: 

\begin{align*}
\grad \gamma &\sum_{t = 0}^{\infty} \gamma^{t} I_{\mathrm{NWJ}}^{f}(s_{t+1}; s_{t} | \theta) \\
&:= \grad \gamma \sum_{t = 0}^{\infty} \gamma^{t} \Big(\E_{p_{\theta, t:t+1}(s_{t+1}, s_{t})}[f(s_{t+1}, s_{t})] - e^{-1}\E_{s_{t+1} \sim p_{\theta, t+1}(s_{t+1})}[\E_{\tilde{s}_{t} \sim p_{\theta, t}(s_{t})}[e^{f(s_{t+1}, \tilde{s}_{t})}]] \Big) \\
&= \grad \gamma \sum_{t = 0}^{\infty} \gamma^{t} \Big(\E_{s_{0}, a_{0}, ... \sim \pi_{\theta}}[f(s_{t+1}, s_{t})] - e^{-1}\E_{s_{0}, a_{0}, ... \sim \pi_{\theta}}[\E_{\tilde{s}_{0}, \tilde{a}_{0}, ... \sim \pi_{\theta}}[e^{f(s_{t+1}, \tilde{s}_{t})}]] \Big) \\
&= \grad \E_{s_{0}, a_{0}, ... \sim \pi_{\theta}}[\sum_{t = 0}^{\infty} \gamma^{t+1} f(s_{t+1}, s_{t})] - \frac{e}{\gamma} \grad \E_{s_{0}, a_{0}, ... \sim \pi_{\theta}}[\sum_{t = 0}^{\infty} \gamma^{t} \E_{\tilde{s}_{0}, \tilde{a}_{0}, ... \sim \pi_{\theta}}[e^{f(s_{t+1}, \tilde{s}_{t})}]] \Big) \numberthis \label{eq:mi1}
\end{align*}

The first term is the gradient of a discounted model-free RL objective with $\bar{r}(s_{t}, a_{t}, s_{t+1}) = f(s_{t+1}, s_{t})$ as the fixed reward function. The second term is not yet a model-free RL objective since the inner expectation explicitly depends on $\theta$. We further expand the second term. 

\begin{align*}
\nabla_{\theta} &\E_{s_{0}, a_{0}, ... \sim \pi_{\theta}}[\sum_{t = 0}^{\infty} \gamma^{t} \E_{\tilde{s}_{0}, \tilde{a}_{0}, ... \sim \pi_{\theta}}[e^{f(s_{t+1}, \tilde{s}_{t})}]]  \\
&= \E_{s_{0}, a_{0}, ... \sim \pi_{\theta}}[\Big(\sum_{t = 0}^{\infty} \nabla_{\theta} \log \pi_{\theta}(a_{t} | s_{t})\Big) \sum_{t = 0}^{\infty} \gamma^{t} \E_{\tilde{s}_{0}, \tilde{a}_{0}, ... \sim \pi_{\theta}}[e^{f(s_{t+1}, \tilde{s}_{t})}]] \\
&\quad\quad+\E_{s_{0}, a_{0}, ... \sim \pi_{\theta}}[\nabla_{\theta} \sum_{t = 0}^{\infty} \gamma^{t} \E_{\tilde{s}_{0}, \tilde{a}_{0}, ... \sim \pi_{\theta}}[e^{f(s_{t+1}, \tilde{s}_{t})}]] \\
&= \E_{s_{0}, a_{0}, ... \sim \pi_{\theta}}[\Big(\sum_{t = 0}^{\infty} \nabla_{\theta} \log \pi_{\theta}(a_{t} | s_{t})\Big) \sum_{t = 0}^{\infty} \gamma^{t} \E_{\tilde{s}_{0}, \tilde{a}_{0}, ... \sim \pi_{\theta}}[e^{f(s_{t+1}, \tilde{s}_{t})}]] \\
&\quad\quad+\E_{s_{0}, a_{0}, ... \sim \pi_{\theta}}[\nabla_{\theta} \E_{\tilde{s}_{0}, \tilde{a}_{0}, ... \sim \pi_{\theta}}[\sum_{t = 0}^{\infty} \gamma^{t} e^{f(s_{t+1}, \tilde{s}_{t})}]] \\
&= \E_{s_{0}, a_{0}, ... \sim \pi_{\theta}}[\Big(\sum_{t = 0}^{\infty} \nabla_{\theta} \log \pi_{\theta}(a_{t} | s_{t})\Big) \sum_{t = 0}^{\infty} \gamma^{t} \E_{\tilde{s}_{0}, \tilde{a}_{0}, ... \sim \pi_{\theta}}[e^{f(s_{t+1}, \tilde{s}_{t})}]] \\
&\quad\quad+\E_{s_{0}, a_{0}, ... \sim \pi_{\theta}}[\E_{\tilde{s}_{0}, \tilde{a}_{0}, ... \sim \pi_{\theta}}[\Big(\sum_{t = 0}^{\infty} \nabla_{\theta} \log \pi_{\theta}(\tilde{a}_{t} | \tilde{s}_{t})\Big) \sum_{t = 0}^{\infty} \gamma^{t} e^{f(s_{t+1}, \tilde{s}_{t})}]] \\
&= \E_{s_{0}, a_{0}, ... \sim \pi_{\theta}}[\Big(\sum_{t = 0}^{\infty} \nabla_{\theta} \log \pi_{\theta}(a_{t} | s_{t})\Big) \sum_{t = 0}^{\infty} \gamma^{t} \E_{\tilde{s}_{0}, \tilde{a}_{0}, ... \sim \pi_{\theta}}[e^{f(s_{t+1}, \tilde{s}_{t})}]] \\
&\quad\quad+\E_{\tilde{s}_{0}, \tilde{a}_{0}, ... \sim \pi_{\theta}}[\E_{s_{0}, a_{0}, ... \sim \pi_{\theta}}[\Big(\sum_{t = 0}^{\infty} \nabla_{\theta} \log \pi_{\theta}(\tilde{a}_{t} | \tilde{s}_{t})\Big) \sum_{t = 0}^{\infty} \gamma^{t} e^{f(s_{t+1}, \tilde{s}_{t})}]] \\
&= \E_{s_{0}, a_{0}, ... \sim \pi_{\theta}}[\Big(\sum_{t = 0}^{\infty} \nabla_{\theta} \log \pi_{\theta}(a_{t} | s_{t})\Big) \sum_{t = 0}^{\infty} \gamma^{t} \E_{\tilde{s}_{0}, \tilde{a}_{0}, ... \sim \pi_{\theta}}[e^{f(s_{t+1}, \tilde{s}_{t})}]] \\
&\quad\quad+\E_{\tilde{s}_{0}, \tilde{a}_{0}, ... \sim \pi_{\theta}}[\Big(\sum_{t = 0}^{\infty} \nabla_{\theta} \log \pi_{\theta}(\tilde{a}_{t} | \tilde{s}_{t})\Big) \sum_{t = 0}^{\infty} \gamma^{t} \E_{s_{0}, a_{0}, ... \sim \pi_{\theta}}[e^{f(s_{t+1}, \tilde{s}_{t})}]] \\
&= \E_{s_{0}, a_{0}, ... \sim \pi_{\theta}}[\Big(\sum_{t = 0}^{\infty} \nabla_{\theta} \log \pi_{\theta}(a_{t} | s_{t})\Big) \sum_{t = 0}^{\infty} \gamma^{t} \Big(\E_{\tilde{s}_{0}, \tilde{a}_{0}, ... \sim \pi_{\theta}}[e^{f(s_{t+1}, \tilde{s}_{t})}] + \E_{\tilde{s}_{0}, \tilde{a}_{0}, ... \sim \pi_{\theta}}[e^{f(\tilde{s}_{t+1}, s_{t})}]\Big)] \\
&= \E_{s_{0}, a_{0}, ... \sim \pi_{\theta}}[\Big(\sum_{t = 0}^{\infty} \nabla_{\theta} \log \pi_{\theta}(a_{t} | s_{t})\Big) \sum_{t = 0}^{\infty} \gamma^{t} \Big(\E_{\tilde{s}_{t} \sim p_{\theta, t}}[e^{f(s_{t+1}, \tilde{s}_{t})}] + \E_{\tilde{s}_{t+1} \sim p_{\theta, t+1}}[e^{f(\tilde{s}_{t+1}, s_{t})}]\Big)]
\\
&= \nabla_{\theta} \E_{s_{0}, a_{0}, ... \sim \pi_{\theta}}[\sum_{t = 0}^{\infty} \gamma^{t} \Big(\E_{\tilde{s}_{t} \sim q_{t}}[e^{f(s_{t+1}, \tilde{s}_{t})}] + \E_{\tilde{s}_{t+1} \sim q_{t+1}}[e^{f(\tilde{s}_{t+1}, s_{t})}]\Big)]
\numberthis \label{eq:mi2}\\
\end{align*}

Combining the results of Eq. \ref{eq:mi1} and Eq. \ref{eq:mi2}, we see that: 
\begin{align*}
    \grad \gamma \sum_{t = 0}^{\infty} \gamma^{t} I_{\mathrm{NWJ}}^{f}(s_{t+1}; s_{t} | \theta) = \grad J(\theta, \bar{r}=r_{f})
\end{align*}
where, $r_{f}(s_{t}, a_{t}, s_{t+1}) = \gamma f(s_{t}, s_{t+1}) -  \frac{\gamma}{e}\E_{\tilde{s}_{t} \sim q_{t}, \tilde{s}_{t+1} \sim q_{t+1}}[e^{f(s_{t+1}, \tilde{s}_{t})} + e^{f(\tilde{s}_{t+1}, s_{t})}]$. 
Finally, putting everything together with the result of Eq. \ref{eq:ent_gradient}: 
\begin{align*}
    \grad \gH^{f}(\rho_{\pi_{\theta}}) = \grad J(\theta, \bar{r}=r_{\pi} + r_{f})
\end{align*}
as desired. 
\end{proof}

\subsection{Discussion of non-adversarial objective}
\label{sec:adversarial_discussion}

Here we briefly elaborate on why the NDI objective is non-adversarial. Recall that for a fixed critic $f: \S \times \S \rightarrow \RR$ the objective is to maximize the expected expert density in addition to the SAELBO $\gH^f$:
\begin{align*}
\max_{\pi_\theta} J(\pi_\theta, \bar{r}=\log \rho_{\pi_{E}}) + \gH^{f}(\rho_{\pi_{\theta}})
\end{align*}
Now suppose we further maximize the SAELBO with respect to the critic $f$
\begin{align*}
\max_{\pi_\theta, f} J(\pi_\theta, \bar{r}=\log \rho_{\pi_{E}}) + \gH^{f}(\rho_{\pi_{\theta}})
\end{align*}
Since \emph{both $\pi_\theta$ and $f$ seek to maximize the lower-bound}, the optimization problem can be solved by coordinate descent where we alternate between updating $f$ and $\pi_{\theta}$ while fixing the counterpart. This corresponds to alternating between policy updates and critic updates where the critic is updated to match the optimal critic for the policy $\pi_{\theta}$, i.e $f^*(s_t, s_{t+1}; \theta) = \log \frac{p_{\theta, t:t+1}(s_t, s_t+1)}{p_{\theta, t}(s_t)p_{\theta, t+1}(s_{t+1})} + 1$. This is in stark contrast to minmax optimization problems for which coordinate descent is not guaranteed to converge \cite{jin2019minmax}. 

\section{Implementation details}
\label{sec:baselines}
Here, we provide implementation details for each IL algorithm.

\textbf{NDI (Ours)}: We experiment with two variants of our method NDI+MADE and NDI+EBM, where the only difference lies in the what density estimation method was used. Across all experiments, our density model $q_{\phi}$ is a two-layer MLP with 256 hidden units and tanh activations. We add spectral normalization \cite{miyato2018spectral} to all layers. All density models are trained with Adam \cite{kingma2014adam} using a learning rate of $0.0001$ and batchsize $256$. We train both MADE and EBM for 200 epochs. All other hyperparameters related to MADE \cite{germain2015made} and SSM \cite{song2019sliced} were taken to be the default values provided in the open-source implementations\footnote{MADE: \url{https://github.com/kamenbliznashki/normalizing_flows}), SSM: \url{https://github.com/ermongroup/ncsn}}. For hyperparameters related to the MaxOccEntRL step, $\lambda_{\pi}$ is tuned automatically in the stable-baselines implementation \cite{stable-baselines}, and we set $\lambda_{f} = 0.005$. All RL related hyperparameters including the policy architecture are the same as those in the original SAC implementation \cite{haarnoja2018soft}. We will be open-sourcing our implementation in the near future. 

\textbf{Behavioral Cloning} \cite{pomerleau1991efficient}: For BC, we use the stable-baselines \cite{stable-baselines} of the \href{https://github.com/hill-a/stable-baselines/blob/master/stable_baselines/common/base_class.py#L289}{.pretrain()} function. We parameterize the model with a two-layer MLP with 256 hidden units. We standardize the observations to have zero mean and unit variance (which we found drastically improves performance). We monitor the validation loss and stop training when the validation loss starts ceases to improve. 

\textbf{GAIL} \cite{ho2016generative}: We use the stable-baselines\cite{stable-baselines} implementation of GAIL using a two-layer MLP with 256 hidden units. Hyperparameters are same as those in the original GAIL implementation for Mujoco \cite{ho2016generative}. During training, we monitor the average discriminator reward and stop training when this reward saturates over 40 episodes. 

\textbf{Random Expert Distillation} \cite{wang2020red}: We use the official implementation\footnote{RED: \url{https://github.com/RuohanW/RED}} of Random Expert Distillation \cite{wang2020red} and explicitly set the BC pretraining flag off for all environments for the results in Table \ref{table:imitation} and Table \ref{table:supp_res}. All other hyperparameters associated with the algorithm were set to the default values that were tuned for Mujoco tasks in the original implementation. For each random seed, we sample the required number of expert trajectories and give that as the input expert trajectories to the RED algorithm.

\textbf{ValueDICE} \cite{kostrikov2020valuedice}: We use the original implementation of ValueDICE \cite{kostrikov2020valuedice}\footnote{ValueDICE: \url{https://github.com/google-research/google-research/tree/master/value_dice}}. All hyperparameters associated with the algorithm were set to the default values that were tuned for Mujoco tasks in the original implementation. For each random seed the algorithm randomly sub-samples the required number of expert trajectories which is passed in as a flag. We conducted a hyperparameter search over the replay regularization and the number of updates performed per time step. We vary the amount of replay regularization from 0 to 0.5 and the number of updates per time step from 2 to 10 but stick with the default values as we do not find any consistent improvement in performance across environments.
\section{Additional experiments}
\label{sec:supp_res}
In this section, we provide additional imitation performance results with 25 expert trajectories and also report imitation performance with an added BC pretraining step. 

\subsection{Varying amounts of demonstration data}
\label{sec:more_demo}
Here we present the results when using 25 expert trajectories: 

\begin{table}[h]
    \setlength{\tabcolsep}{4pt}
    \renewcommand{\arraystretch}{0.9}
    \small
    \caption{\textbf{Task Performance} when provided with a 25 demonstrations. NDI outperforms all baselines on all tasks.}
    \vspace{0pt}
    \label{table:supp_res}
    \begin{center}
    \begin{sc}
    \begin{tabular}{lccccc}
    \toprule
    & Hopper & Half-Cheetah & Walker & Ant & Humanoid \\
    \midrule
    Random & $14\pm 8$ & $-282 \pm 80$ & $1 \pm 5$ & $-70 \pm 111$ & $123 \pm 35$ \\
    \midrule \midrule
    \midrule
    BC & $3498 \pm 103$ & $4167 \pm 95$ & $4816 \pm 196$ & $3596 \pm 214$ & $4905 \pm 612$\\
    RED & $2523 \pm 476$ & $ -3 \pm 4$ & $1318 \pm 446$ & $1004 \pm 5$ & $2052 \pm 585$\\
    GAIL & $3521 \pm 44 $ & $3632 \pm 225$ & $4926 \pm 450$ & $ 3582 \pm 212$ & $259 \pm 21$\\
    ValueDICE & $ 2829 \pm 685$ & $4105 \pm 134$ & $ 4384 \pm 620$ & $3948 \pm 350$ & $ 2116 \pm 1005$\\
    \rowcolor{LightOrange} NDI+MADE & $3514 \pm 105$ & $4253 \pm 105$ & $4892 \pm 109$ & $1023 \pm 322$ & $6013 \pm 550$ \\
    \rowcolor{LightOrange} NDI+EBM  & $\mathbf{3557 \pm 109}$ & $\mathbf{5718 \pm 294}$ & $\mathbf{5210 \pm 105}$ & $\mathbf{4319 \pm 107}$ & $\mathbf{6113 \pm 210}$ \\
    \midrule \midrule
    Expert & $3567 \pm 4$ & $4142 \pm 132$ & $5006 \pm 472$ & $4362 \pm 827$ & $5417 \pm 2286$ \\
    \bottomrule
    \end{tabular}
    \end{sc}
    \end{center}
\end{table}

NDI+EBM outperforms all other methods. RED is still unable to perform well on all tasks. We found that even after hyperparameter tuning, ValueDICE and GAIL slightly underperform the expert on some tasks. 

\subsection{Pretraining with BC}
\label{sec:pretrain}
Here we present the results obtained by pretraining all algorithms with Behavioral Cloning using 1 expert trajectory. The number of pretraining epochs was determined separately for each baseline algorithm through a simple search procedure. 

For RED, we found 100 to be the optimal number of pretraining epochs and more pretraining worsens performance. For GAIL, we use 200 pretraining epochs after conducting a search from 100 to 1000 epochs. We find that the performance improves till 200 epochs and pretraining any longer worsens the performance. For ValueDICE, we use 100 pretraining epochs, determined by the same search procedure as GAIL, and found that the performance decreases when using more than 200 pretraining epochs. For NDI+MADE and NDI+EBM, we use 100 pretraining epochs.

\begin{table}[h]
    \small
    \caption{\textbf{Task Performance} when pretrained with BC and provided with 1 (top), 25 (bottom) expert demonstration.}
    \label{table:supp_res_BC1}
    \begin{center}
    \begin{sc}
    \begin{tabular}{lccccc}
    \toprule
    & Hopper & Half-Cheetah & Walker & Ant & Humanoid \\
    \midrule
    Random & $14\pm 8$ & $-282 \pm 80$ & $1 \pm 5$ & $-70 \pm 111$ & $123 \pm 35$ \\
    \midrule\midrule
    \multicolumn{6}{c}{1 Demonstrations} \\
    \midrule
    RED & $3390 \pm 197$ & $ 3267 \pm 614$ & $2260 \pm 686$ & $3044 \pm 612$ & $571 \pm 191$\\
    GAIL & $3500 \pm 81 $ & $3350 \pm 512$ & $4175 \pm 825$ & $ 2716 \pm 210$ & $221 \pm 48$\\
    ValueDICE & $ 1507 \pm 308$ & $3556 \pm 247$ & $ 1937 \pm 912$ & $1007 \pm 94$ & $ 372 \pm 31$\\
    \rowcolor{LightOrange} NDI+MADE & $3526 \pm 172$ & $4152 \pm 209$ & $4998 \pm 157$ & $4014 \pm 105$ & $5971 \pm 550$ \\
    \rowcolor{LightOrange} NDI+EBM & $\mathbf{3589 \pm 32}$ & $\mathbf{4622 \pm 210}$ & $\mathbf{5105 \pm 105}$ & $\mathbf{4412 \pm 204}$ & $5606 \pm 314$ \\
    \midrule\midrule
    \multicolumn{6}{c}{25 Demonstrations} \\
    \midrule
    RED & $3460 \pm 153$ & $ 3883 \pm 440$ & $4683 \pm 994$ & $4079 \pm 208$ & $4385 \pm 1725$\\
    GAIL & $3578 \pm 24$ & $4139 \pm 275$ & $4904 \pm 282$ & $ 3534 \pm 346$ & $281 \pm 50$\\
    ValueDICE & $2124 \pm 628$ & $3975 \pm 125$ & $ 3939 \pm 1152$ & $3559 \pm 134$ & $101 \pm 33$\\
    \rowcolor{LightOrange} NDI+MADE & $\mathbf{3533 \pm 130}$ & $4210 \pm 159$ & $5010 \pm 189$ & $4102 \pm 99$ & $5103 \pm 789$ \\
    \rowcolor{LightOrange} NDI+EBM & $3489 \pm 73$ & $\mathbf{4301 \pm 155}$ & $\mathbf{5102 \pm 77}$ & $\mathbf{4201 \pm 153}$ & $\mathbf{5501 \pm 591}$ \\
    \midrule\midrule
    Expert & $3567 \pm 4$ & $4142 \pm 132$ & $5006 \pm 472$ & $4362 \pm 827$ & $5417 \pm 2286$ \\
    \bottomrule
    \end{tabular}
    \end{sc}
    \end{center}
\end{table}

We observe that the performance for RED improves drastically when pretrained with BC but is still unable to achieve expert level performance when given 1 demonstration. We observe GAIL produces better results when pretrained for 200 epochs. ValueDICE does not seem to benefit from pretraining. Pretraining also slightly improves the performance of NDI, notably in boosting the performance of NDI+MADE on Ant. 

\subsection{Environment sample complexity}
\label{sec:environment_sample}

Although minimizing environment interactions is not a goal of this work, we show these results in Table \ref{table:environment_sample} for completeness.  We found that NDI roughly requires an order of magnitude less samples than GAIL which may be attributed to using a more stable non-adversarial optimization procedure. ValueDICE, an off-policy IL algorithm optimized to minimize environment sample complexity, requires roughly two orders of magnitude less interactions than NDI. We hope to see future work combine off-policy RL algorithms with NDI to further reduce environment interactions.

\begin{table}[h]
    \setlength{\tabcolsep}{4pt}
    \renewcommand{\arraystretch}{0.9}
    \small
    \caption{\textbf{Environment Sample Complexity} computed as the mean number of environment steps needed to reach expert level performance when provided with ample (25) expert demonstrations. RED excluded as it cannot reach expert performance without BC pretraining. NDI requires less samples than GAIL but more than ValueDICE.}
    \vspace{-5pt}
    \label{table:environment_sample}
    \begin{center}
    \begin{sc}
    \begin{tabular}{lccccc}
    \toprule
    & Hopper & Half-Cheetah & Walker & Ant & Humanoid \\
    \midrule
    GAIL & $8.9\mathrm{M} \pm 1.3\mathrm{M}$ & $10.0\mathrm{M} \pm 3.3\mathrm{M}$ & $15.1\mathrm{M} \pm 3.5\mathrm{M}$ & $34.2\mathrm{M} \pm 8.8\mathrm{M}$ & $43.2\mathrm{M} \pm 11.2\mathrm{M}$ \\
    ValueDICE & $8.3\mathrm{K} \pm 1.6\mathrm{K}$ & $10.7\mathrm{K} \pm 2.1\mathrm{K}$ & $24.3\mathrm{K} \pm 4.5\mathrm{K}$ & $6.9\mathrm{K} \pm 1.1\mathrm{K}$ & $105\mathrm{K} \pm 10.2\mathrm{K}$\\
    \rowcolor{LightOrange} NDI+MADE & $0.8\mathrm{M} \pm 0.2\mathrm{M}$ & $1.3\mathrm{M} \pm 0.4\mathrm{M}$ & $4.8\mathrm{M} \pm 1.1\mathrm{M}$ & $4.5\mathrm{M} \pm 0.5\mathrm{M}$ & $6.8\mathrm{M} \pm 1.7\mathrm{M}$ \\
    \rowcolor{LightOrange} NDI+EBM & $0.5\mathrm{M} \pm 0.1\mathrm{M}$ & $1.4\mathrm{M} \pm 0.3\mathrm{M}$ & $4.1\mathrm{M} \pm 2.1\mathrm{M}$ & $4.9\mathrm{M} \pm 1.5\mathrm{M}$ & $6.1\mathrm{M} \pm 1.1\mathrm{M}$ \\
    \bottomrule
    \end{tabular}
    \end{sc}
    \end{center}
\vspace{-10pt}
\end{table}

\end{document}